\documentclass{article}
\usepackage{iclr2027_conference,times}
\usepackage{url}
\usepackage{graphicx}
\usepackage{subcaption}
\usepackage{amsmath,amssymb,mathtools,amsthm}
\usepackage{algorithm}
\usepackage{algpseudocode}
\usepackage{booktabs}
\usepackage{multirow}
\usepackage{xspace}
\usepackage{xcolor}
\usepackage[hidelinks]{hyperref}
\newtheorem{theorem}{Theorem} \newtheorem{proposition}[theorem]{Proposition}  \newtheorem{corollary}[theorem]{Corollary} \theoremstyle{definition} \newtheorem{definition}[theorem]{Definition} \newtheorem{assumption}[theorem]{Assumption} \theoremstyle{remark} \newtheorem{remark}[theorem]{Remark}
\newtheorem*{remark*}{Remark}

\theoremstyle{definition}
\newtheorem{regularity}{Condition}

\newcommand{\Nzero}{\mathbb{N}_0}
\newcommand{\pa}{\operatorname{Pa}}
\newcommand{\ch}{\operatorname{Ch}}
\newcommand{\Var}{\operatorname{Var}}

\newcommand{\logit}{\operatorname{logit}}

\newcommand{\E}{\mathbb{E}}
\newcommand{\Dop}{\mathcal{D}}
\newcommand{\DISCO}{\textsc{DISCO}\xspace}
\newcommand{\DISCOCont}{\textsc{DISCO-Cont}\xspace}
\newcommand{\DISCONoRank}{\textsc{DISCO-NoRank}\xspace}

\newcommand{\ODSpoi}{\textsc{ODS} (Poi)\xspace}

\newcommand{\NOTEARSls}{\textsc{NOTEARS-MLP}\xspace}

\newcommand{\TableStyle}{%
  \centering
  \normalsize
  \setlength{\tabcolsep}{4pt}%
  \renewcommand{\arraystretch}{1.08}%
}


\makeatletter
\newcommand{\AppendixContents}{%
  \clearpage
  \begingroup
    \renewcommand{\contentsname}{Appendix Contents}%
    \setlength{\parskip}{0pt}%
    \hypersetup{linktoc=all}%
    \let\ContentsRef\ref
    \renewcommand*{\ref}[1]{\ContentsRef*{##1}}%
    \renewcommand*{\l@section}[2]{%
      \ifnum\value{tocdepth}>0\relax
        \addvspace{3pt}%
        \@dottedtocline{1}{0em}{1.5em}{\bfseries ##1}{\bfseries ##2}%
      \fi
    }%
    \pdfbookmark[0]{Appendix Contents}{appendix-contents}%
    \tableofcontents
  \endgroup
  \clearpage
}
\makeatother

\title{Discrete Score Matching Enables\\Causal Discovery from Count Data}

\author{%
Euijong Song\qquad Hyewon Park\qquad
Gunwoong Park\thanks{Interdisciplinary Program in Artificial Intelligence; Institute for Data Innovation in Science, Seoul National University, Korea.}\\
\normalfont Department of Statistics\\
\normalfont Seoul National University, Korea\\
\normalfont\texttt{brejsong@gmail.com}\quad
\texttt{hwonpark18@gmail.com}\quad\texttt{gw.park23@gmail.com}}

\iclrfinalcopy
\hypersetup{
  pdftitle={Discrete Score Matching Enables Causal Discovery from Count Data},
  pdfauthor={Euijong Song, Hyewon Park, Gunwoong Park}
}

\begin{document}
\raggedbottom
\addtocontents{toc}{\protect\setcounter{tocdepth}{-1}}

\maketitle
\fancyhead{}
\renewcommand{\headrulewidth}{0.4pt}

\begin{abstract}
Count data pose a challenge for score-matching-based causal discovery: derivatives are unavailable, and simply replacing them with finite differences does not generally suffice for causal discovery. We generalize SCORE's constant-curvature criterion \citep{rolland2022score} by conditioning on the node's value, yielding the conditional curvature score (CCS) for ordering. We also extend curvature-based parent recovery through the off-diagonal curvature score (OCS), enabling directed acyclic graph (DAG) recovery with both scores constructed from score functions for continuous data and concrete scores for counts. In the bivariate setting, zero CCS exactly characterizes a semiparametric generalized linear model (GLM) conditional form in which the conditional family need not be specified in advance, unlike in classical GLMs. For bivariate semiparametric GLM DAGs under our regularity condition, canonical-parameter nonlinearity is necessary and sufficient for identifiability. In multivariate DAGs, this nonlinearity enables DAG recovery through CCS and OCS.  Our framework identifies a new class of semiparametric GLM DAGs that strictly contains the nonlinear Gaussian ANM class identified by SCORE. We introduce \DISCO\ (\emph{DI}screte \emph{SCO}re), a count-DAG recovery algorithm that estimates CCS and OCS using discrete diffusion. Experiments demonstrate accurate DAG recovery across Poisson, negative binomial, binomial, and mixed-family settings, as well as scalability to 1,000-node DAGs on a single GPU.
\end{abstract}

\section{Introduction}

Many causal-discovery problems are naturally count-valued: gene-regulatory analysis from RNA-seq counts \citep{love2014deseq2}, community structure from species abundances \citep{stoklosa2022negbin}, and topic structure from word counts \citep{blei2003lda}. Such counts rarely follow a single distributional family: sequencing studies model expression counts as negative binomial (NB) but allele counts as binomial \citep{love2014deseq2,robinson2010edger,castel2015allelic}, and ecological surveys mix Poisson and overdispersed NB \citep{stoklosa2022negbin}.

Order-based methods provide a scalable alternative to DAG search by estimating a causal order and then selecting parents \citep{ghoshal2018learning,gao2020polynomial,peters2014cam}. Score-matching-based methods follow this approach using log-density derivatives \citep{rolland2022score,montagna2023das,montagna2023nogam,sanchez2023diffan}. Concrete score matching extends score estimation to discrete data \citep{meng2022concrete}, but finite-difference curvature need not satisfy SCORE's sink criterion: even at a sink, the node's conditional density can contribute curvature that varies with its value. We call this contribution \emph{own-factor curvature}.

Existing count-DAG methods obtain identifiability by exploiting properties of specified conditional families, such as Poisson or known quadratic variance function (QVF) families \citep{park2015poisson,park2018qvf}. More flexible conditionally parametric models allow general parent effects but still specify the family itself \citep{bodik2025cpcm}. Their guarantees consequently depend on choosing appropriate conditional families.  Extending score-matching-based causal discovery to counts without prespecified families therefore requires an ordering criterion that accommodates unknown own-factor curvature.

We address these limitations by generalizing SCORE's constant-curvature criterion through the CCS and extending curvature-based parent recovery through the OCS. Our \emph{conditional curvature score} (CCS) removes the own-factor contribution to curvature variance while retaining variation induced by children. The CCS is obtained from derivatives of the continuous score or finite differences of log concrete-score entries.

We establish identifiability for a new class of semiparametric GLM DAGs. In the bivariate setting, zero CCS exactly characterizes a semiparametric GLM conditional form without prespecifying the conditional family. Within this class, canonical-parameter nonlinearity is necessary and sufficient for identifiability. In multivariate settings, this nonlinearity enables the CCS to recover a causal order, after which the \emph{off-diagonal curvature score} (OCS) recovers parents. On continuous data, the identifiable class strictly contains SCORE's nonlinear Gaussian additive noise models (ANMs).

We introduce \DISCO\ (\emph{DI}screte \emph{SCO}re), a scalable count-DAG recovery algorithm that estimates the CCS and OCS. We first estimate joint concrete scores using score-entropy discrete diffusion, then project them to the marginal concrete scores needed to compute the CCS and OCS. With OCS-based parent selection, a rank transform makes the recovered DAG invariant to strictly increasing coordinate-wise transformations.

Our contributions are as follows.
\begin{itemize}
\item \textbf{Discrete score matching for count-DAG discovery.}
We generalize SCORE's constant-curvature principle to count data through the conditional curvature score (CCS), which identifies sinks despite node-specific own-factor curvature. We also extend curvature-based parent recovery through the OCS.

\item \textbf{A new identifiable class of semiparametric GLM DAGs.}
We derive the semiparametric GLM conditional form from the bivariate zero-CCS condition without prespecifying the family. Canonical-parameter nonlinearity is necessary and sufficient for bivariate identifiability. This nonlinearity also yields multivariate DAG identifiability. In the continuous setting, the resulting identifiable class strictly contains the nonlinear Gaussian ANM class identified by SCORE.

\item \textbf{Scalable count-DAG recovery with discrete diffusion.}
We develop \DISCO{}, which uses score-entropy discrete diffusion to estimate joint concrete scores and projects them to the marginal concrete scores needed to compute the CCS and OCS. \DISCO{} scales to 1,000-node count DAGs on a single GPU.
\end{itemize}

\section{Background}
\label{sec:background}

\subsection{Related Work}
\label{sec:related}

\paragraph{Score-matching-based causal discovery.}
For a continuous density $p$, the score function is $s(x)=\nabla_x\log p(x)$ \citep{hyvarinen2005estimation}. SCORE identifies sinks in nonlinear Gaussian ANMs through the constant-curvature criterion, $\Var[\partial_j s_j(X)]=0$ \citep{rolland2022score}. Subsequent work allows non-Gaussian noise (NoGAM; \citealp{montagna2023nogam}), improves Hessian estimation (DiffAN and SciNO; \citealp{sanchez2023diffan,kang2025scino}), and uses off-diagonal Hessian components for parent recovery (DAS; \citealp{montagna2023das}).

For discrete causal discovery, \citet{vo2026generalized} use a marginalization-based generalized score \citep{lyu2009generalized} to recover a causal order under a non-decreasing conditional-randomness assumption. In contrast, we use concrete scores \citep{meng2022concrete}, estimated through discrete diffusion with SEDD's score-entropy objective \citep{lou2024discrete}, and establish DAG identifiability within our semiparametric GLM class under the stated nonlinearity and regularity conditions, without requiring a cross-node ordering of conditional randomness.

\paragraph{Count-DAG identifiability.}
Count-DAG identifiability typically relies on distribution-specific conditional structure: Poisson and QVF DAGs assume known conditional variance structure \citep{park2015poisson,park2019mrs,park2018qvf}, while conditionally parametric causal models still prespecify the conditional family \citep{bodik2025cpcm}. Our semiparametric GLM DAG class leaves conditional families, links, and canonical-parameter functions unspecified.

\subsection{Notation}
\label{sec:prelim}

Let $G=(V,E)$ be a directed acyclic graph with $V=[d]:=\{1,\dots,d\}$. Let $\pa_G(j)$ and $\ch_G(j)$ denote the parents and children of node $j$ in $G$. A \emph{source} is a node $j$ with $\pa_G(j)=\emptyset$, a \emph{non-source} is a node $j$ with $\pa_G(j)\neq\emptyset$, and a \emph{sink} is a node $j$ with $\ch_G(j)=\emptyset$. For $R\subseteq V$, $G_R$ denotes the induced subgraph of $G$ on $R$. A \emph{causal order} is a permutation $\pi:V\to[d]$ with $\pi(j)<\pi(i)$ for every edge $j\to i$; we also call it an ordering of $G$.

Let $P$ be the distribution of $X=(X_1,\dots,X_d)$. For each $j\in V$, the coordinate $X_j$ takes values in $\mathcal X_j\subseteq\mathbb R$ and has a dominating measure $\nu_j$ (counting or Lebesgue). With $\nu=\bigotimes_{j=1}^d\nu_j$, write $p=dP/d\nu$ for the density; throughout, this term includes probability mass functions. The set $\mathcal X_j$ is the declared effective support. We call $X_j$ a \emph{count coordinate} when $\mathcal X_j$ is a finite or countably infinite set of consecutive integers. Define $\mathring{\mathcal X}_j=\operatorname{int}(\mathcal X_j)$ for a continuous coordinate and $\mathring{\mathcal X}_j=\mathcal X_j$ for a count coordinate.

For $A\subseteq V$, write $X_A=(X_i)_{i\in A}$, $\mathcal X_A=\prod_{i\in A}\mathcal X_i$, and $\mathring{\mathcal X}_A=\prod_{i\in A}\mathring{\mathcal X}_i$. For $j\in A$, $x_A\in\mathcal X_A$, and $y\in\mathcal X_j$, let $x_A^{j\to y}$ be the vector obtained by replacing coordinate $j$ with $y$. For $R\subseteq V$, let $P_R$ denote the marginal distribution of $X_R$ and $p_R$ its density.

We use a single coordinate operator $\Dop_j$ for both continuous and count coordinates. With $e_j$ the $j$th unit vector,
\[
\bigl(\Dop_j f,\ \Dop_j^2 f\bigr)(x)=
\begin{cases}
\bigl(\partial_{x_j} f(x),\ \partial_{x_j}^2 f(x)\bigr), & X_j\ \text{continuous},\\[3pt]
\bigl(f(x{+}e_j)-f(x),\ f(x{+}2e_j)-2f(x{+}e_j)+f(x)\bigr), & X_j\ \text{count}.
\end{cases}
\]
On a count coordinate these are \emph{partial} operators: every expectation or conditional variance involving a finite-difference curvature is restricted to states where all required shifted values remain in the support. Appendix~\ref{app:notation} gives the formal domains.

\section{Curvature Scores}
\label{sec:ccs}

We introduce two curvature scores: the \emph{conditional curvature score} (CCS) for ordering and the \emph{off-diagonal curvature score} (OCS) for parent recovery. Define the diagonal and off-diagonal log-density curvatures $H_j:=\Dop_j^2\log p$ and $H_{ij}:=\Dop_i\Dop_j\log p$ for $i\neq j$. For the marginal distribution $P_R$, we denote the corresponding curvatures by $H_{j,R}$ and $H_{ij,R}$.

For node $j$, we ask whether its log-density curvature still varies with the other variables after fixing $X_j$. The conditional curvature score measures this remaining variation.

\begin{definition}[Conditional curvature score]
\label{def:ccs}
For node $j$ with $H_j\in L^2$,
\begin{equation}
\label{eq:ccs-def}
S_j=\E\bigl[\Var\{\,H_{j}(X)\mid X_j\,\}\bigr].
\end{equation}
\end{definition}

For the marginal distribution $P_R$, write $S_j(R)=\E_{P_R}[\Var_{P_R}\{H_{j,R}(X_R)\mid X_j\}]$, so $S_j=S_j(V)$, with the same admissible-domain convention. The CCS vanishes exactly when $H_j(X)$ is almost surely a function of $X_j$ alone. By contrast, the \emph{constant-curvature criterion} requires $H_j(X)$ to be almost surely constant, equivalently $\Var[H_j(X)]=0$. For continuous data, $H_j=\partial_j s_j$ is a diagonal component of the log-density Hessian, and this is SCORE's sink criterion. The CCS therefore allows curvature to vary across values of $X_j$. Section~\ref{sec:identifiability} establishes when the CCS separates sinks from nonsinks.

To define the OCS, fix a causal order $\pi$ of $G$. For each node $i$, define its predecessors by $\mathrm{Pred}(i)=\{j:\pi(j)<\pi(i)\}$ and let $R_i=\{i\}\cup\mathrm{Pred}(i)$. The OCS is computed under the marginal distribution $P_{R_i}$.

\begin{definition}[Off-diagonal curvature score]
\label{def:ocs}
For a node $i$ and a predecessor $j\in\mathrm{Pred}(i)$ with $H_{ij,R_i}\in L^1$,
\begin{equation}
\label{eq:parent-score}
\mathrm{OCS}_{j\to i}
=
\E\bigl[\lvert H_{ij,R_i}(X_{R_i})\rvert\bigr].
\end{equation}
\end{definition}

For count coordinates, $H_{ij,R_i}$ records how the log concrete-score entry for a shift in coordinate $j$ changes when $X_i$ steps by one level. The OCS averages its absolute magnitude under $P_{R_i}$.

\section{Semiparametric GLM DAGs}
\label{sec:semiparametric-glm-dags}

A classical generalized linear model (GLM) specifies a conditional family and a link function that maps the conditional mean to a linear predictor \citep{nelder1972generalized,mccullagh1989generalized}. Semiparametric GLMs relax the requirement to prespecify the conditional family \citep{rathouz2009generalized,yang2018semiparametric}. This conditional form naturally arises from the zero-CCS condition in Theorem~\ref{thm:bivariate-sink}, motivating the DAG class introduced below.

\begin{definition}[Semiparametric GLM DAG]
\label{def:ef-dag}
A distribution $P$ follows a \emph{semiparametric GLM DAG} on $G$ if its density factorizes as $p(x)=\prod_{j=1}^d p_j(x_j\mid x_{\pa_G(j)})$. Source marginal densities $p_j$ on $\mathcal X_j$ are arbitrary and need not belong to an exponential family. For every non-source node $j$, its conditional density has the form
\begin{equation}
\label{eq:ef-density}
p_j(x_j\mid x_{\pa_G(j)})
=
b_j(x_j)\exp\bigl\{x_j\,\eta_j(x_{\pa_G(j)})
-A_j\bigl(\eta_j(x_{\pa_G(j)})\bigr)\bigr\},
\end{equation}
where $\eta_j$ is an unspecified canonical-parameter function and $A_j(t):=\log\int_{\mathcal{X}_j}b_j(u)e^{ut}\,d\nu_j(u)$. The support $\mathcal X_j$ and base function $b_j$ are unspecified but do not depend on the parent values. Additionally, each $\eta_j$ depends nontrivially on every parent.
\end{definition}

The support $\mathcal X_j$ and base function $b_j$ determine the conditional family, while $\eta_j$ captures all parent dependence. The canonical-parameter function $\eta_j$ can be represented through a strictly increasing link $g_j$ and an index $f_j$:
\[
g_j\!\left(\E[X_j\mid X_{\pa_G(j)}]\right)
=f_j(X_{\pa_G(j)}),
\qquad
f_j:=g_j\circ A_j'\circ\eta_j.
\]
This link--index pair is nonunique, so we state all structural assumptions directly in terms of $\eta_j$.

\begin{definition}[Nonlinear semiparametric GLM DAG]
\label{def:nonlinear-ef-dag}
A semiparametric GLM DAG is \emph{nonlinear} if, for every edge $k\to i$, there are fixed values $z$ of the other parents in the product-support interior such that $\eta_i(\cdot,z)$ is nonlinear on $\mathring{\mathcal X}_k$. Equivalently, $\Dop_k^2\eta_i\not\equiv0$ on the corresponding derivative or finite-difference domain.
\end{definition}

For Gaussian conditionals with fixed variance, canonical-parameter nonlinearity coincides with the mechanism nonlinearity used by SCORE \citep{rolland2022score}. Definition~\ref{def:nonlinear-ef-dag} extends this restriction to other conditional families. Parent-dependent nuisance parameters lie outside the class, for example a negative binomial whose size depends on the parents.

\begin{table}[H]
\TableStyle
\caption{Example conditional distributions, with index $f=f_j(x_{\pa_G(j)})$ and logistic function $\operatorname{sigmoid}$. Noncanonical links can yield nonlinear $\eta$ even for affine $f$. NB and Gamma use the mean as their second parameter, with fixed size $r$ and shape $\alpha$, respectively.}
\label{tab:examples}
\begin{tabular}{@{}llcc@{}}
\toprule
Conditional $X_j\mid x_{\pa_G(j)}$ & \shortstack{Conditional\\family} & Support & Link $g$ \\
\midrule
$\mathrm{Poisson}(\exp f)$            & Poisson         & $\mathbb N_0$        & $\log$ \\
$\mathrm{Poisson}(\operatorname{softplus} f)$ & Poisson & $\mathbb N_0$        & $\operatorname{softplus}^{-1}$ \\
$\mathrm{NB}(r,\exp f)$               & NB$(r)$         & $\mathbb N_0$        & $\log$ \\
\addlinespace[2pt]
$\mathrm{Binomial}(M,\operatorname{sigmoid}(f))$ & Binomial$(M)$ & $\{0,\dots,M\}$ & $\mu\mapsto\logit(\mu/M)$ \\
$\mathcal N(f,\tau^2)$              & Gaussian        & $\mathbb R$          & identity \\
$\mathrm{Gamma}(\alpha,\exp f)$       & Gamma$(\alpha)$ & $(0,\infty)$         & $\log$ \\
\bottomrule
\end{tabular}
\end{table}

\paragraph{Regularity conditions.} We assume strictly positive densities on product-support interiors and at least three support values for count coordinates. The results in Section~\ref{sec:identifiability} are stated under the corresponding regularity conditions in Appendix~\ref{app:standing}.

\section{Curvature Characterization and DAG Identifiability}
\label{sec:identifiability}

\subsection{Bivariate Characterization and Identifiability}
\label{sec:bivariate}

For both continuous and count variables, we first characterize the conditional distributions with zero CCS. We then establish when the causal direction is identifiable within the resulting model class.

\begin{theorem}[Bivariate characterization by the CCS]
\label{thm:bivariate-sink}
$S_Y=0$ if and only if $p(y\mid x)$ has the semiparametric GLM conditional form in \eqref{eq:ef-density}.
\end{theorem}

If $S_Y=0$, the log-density curvature in $y$ depends only on $y$. Integrating derivatives or summing finite differences gives $\log p(x,y)=a(x)+\log b_Y(y)+y\eta(x)$, which yields this conditional form after normalization.

\begin{theorem}[Bivariate identifiability characterization]
\label{thm:link-index}
Let $P$ have a semiparametric GLM representation with edge $X\to Y$, with Condition~\ref{cond:affine-reverse} (affine reversibility; Appendix~\ref{app:glm-regularity}) imposed in the affine case. Within this model class,
\[
X\to Y\ \text{is identifiable}
\quad\Longleftrightarrow\quad
\eta\ \text{is nonlinear on}\ \mathring{\mathcal X}_X.
\]
\end{theorem}

To relate this characterization to the CCS, consider $X\to Y$. The child conditional satisfies
\[
\log p(y\mid x)
=\underbrace{\log b_Y(y)}_{\text{family-specific term}}
+\underbrace{y\eta(x)-A_Y(\eta(x))}_{\text{parent contribution, affine in }y}.
\]
Taking a second derivative or second difference in $y$ removes the affine parent contribution, leaving $H_Y=\kappa_Y(Y)$, where $\kappa_Y(y):=\Dop_Y^2\log p(y\mid x)=\Dop_Y^2\log b_Y(y)$ is the child's own-factor curvature. This curvature need not be constant: for a negative binomial conditional with fixed size $r>0$, $H_Y=\log\frac{(Y+r+1)(Y+1)}{(Y+r)(Y+2)}$, which is zero for $r=1$ and nonconstant otherwise. Thus, simply replacing derivatives with finite differences in SCORE does not ensure constant curvature at a sink. Conditioning on $Y$, however, gives $S_Y=0$.

For the parent $X$, the curvature separates as
\[
H_X
=\kappa_X(X)
+\underbrace{Y\Dop_X^2\eta(X)-\Dop_X^2[A_Y(\eta(X))]}_{\psi_X(X,Y):\ \text{residual curvature}}.
\]
Here $\kappa_X(x):=\Dop_X^2\log p_X(x)$ is the source's own-factor curvature. Conditioning on $X$ makes all terms except $Y\Dop_X^2\eta(X)$ deterministic. Hence $S_X=\E[\{\Dop_X^2\eta(X)\}^2\Var(Y\mid X)]$. Since $\Var(Y\mid X)>0$, $S_Y=0<S_X$ exactly when $\eta$ is nonlinear.

\begin{remark}[Translation to link and index]
Identifiability is determined by nonlinearity of the canonical-parameter function $\eta$, not by the link or index in isolation. Appendix~\ref{app:link-index} gives the link--index characterization.
\end{remark}

\subsection{Multivariate Identifiability}
\label{sec:multivariate}

We extend the CCS argument to recover a causal order in multivariate DAGs and use the OCS to recover parents. A remaining set $R$ obtained by repeatedly removing sinks is \emph{ancestral}. Its marginal $P_R$ retains the semiparametric GLM factorization on $G_R$ and the nonlinearity condition.

\begin{samepage}
\begin{theorem}[DAG recovery by the CCS and OCS]
\label{thm:sink-identification}
Let $P$ follow a semiparametric GLM DAG on $G$.
\begin{itemize}
\item[(i)] If $P$ follows a nonlinear semiparametric GLM DAG, then, for every ancestral remaining set $R$ and $j\in R$, $j$ is a sink of $G_R$ if and only if $S_j(R)=0$.
\item[(ii)] Given a causal order and $j\in\mathrm{Pred}(i)$, $j\in\pa_G(i)$ if and only if $\mathrm{OCS}_{j\to i}>0$.
\end{itemize}
\end{theorem}
\end{samepage}

On an ancestral remaining set $R$, the bivariate decomposition extends to $H_{j,R}(x_R)=\kappa_j(x_j)+\psi_{j,R}(x_R)$, where $\kappa_j(x_j):=\Dop_j^2\log p_j(x_j\mid x_{\pa_G(j)})$ is the own-factor curvature and $\psi_{j,R}$ is the residual curvature from the child contributions. Conditioning on $X_j$ makes $\kappa_j(X_j)$ deterministic, so the CCS measures the conditional variation of $\psi_{j,R}$. This residual vanishes at a sink; at a nonsink, the nonlinearity condition makes its averaged conditional variance positive.

For parent recovery, node $i$ is a sink of $G_{R_i}$.  Under $P_{R_i}$, $H_{ij,R_i}=\Dop_j\eta_i$ if $j\in\pa_G(i)$ and $H_{ij,R_i}=0$ otherwise. Thus a nonconstant parent effect suffices for OCS; nonlinearity is not required.

\begin{corollary}[DAG identifiability]
\label{cor:dag-identifiability}
Let $P$ follow a nonlinear semiparametric GLM DAG on $G$. Then $G$ is identifiable from $P$ within this nonlinear semiparametric GLM class.
\end{corollary}

This class strictly contains SCORE's nonlinear Gaussian ANMs. For example, the Gamma model in Appendix~\ref{app:score-proofs} is identifiable despite its nonconstant sink curvature $H_Y=-(\alpha-1)/Y^2$, while $S_Y=0$.

\section{\DISCO\ Algorithm}
\label{sec:disco-algorithm}

\DISCO\ consists of concrete-score estimation, CCS-based ordering, and parent selection given the estimated order. After applying the rank transform, we estimate joint concrete scores using score-entropy discrete diffusion and project them to the marginal concrete scores needed for CCS-based ordering and OCS-based parent selection. Parent recovery can use CAM pruning \citep{buhlmann2014cam}, PCM-GAM \citep{lundborg2024projected}, or the OCS-based rule. The OCS-based rule reuses the fitted score networks, providing a scalable alternative to regression- and testing-based parent selection. Appendix~\ref{app:diagnostic-parent-selection} gives detailed comparisons of these methods. Appendix~\ref{app:continuous-recovery} presents a continuous-data version and a comparison with DiffAN \citep{sanchez2023diffan}.

\subsection{Concrete-Score Estimation}
\label{sec:concrete-score-est}
\label{sec:rank-preprocess}
We first apply the rank transform, mapping the distinct training values of each count variable $j$ to the empirical rank grid $\{0,\ldots,K_j\}$, where $K_j+1$ is the number of distinct training values. We then estimate concrete scores on the transformed data using discrete diffusion.

The rank transform makes neighboring differences invariant to strictly increasing transformations on the support of each variable. For estimation, $X$, $\mathcal X_j$, and $p_R$ refer to the transformed data, rank grid, and corresponding marginal density.

For $j\in R\subseteq V$ and $y\in\mathcal X_j$, the concrete-score entry is the probability ratio $s_{j,y}^{R}(x_R):=p_R(x_R^{j\to y})/p_R(x_R)$. On the transformed data, we train a joint concrete-score network $s_\theta(x;\sigma)$ to estimate the concrete score of the noise-corrupted joint distribution by minimizing a neighbor-restricted denoising score-entropy objective \citep{lou2024discrete}. Here $\sigma$ denotes the diffusion noise level. As $\sigma\to0$, the concrete score of the noise-corrupted distribution converges to that of the data distribution.

\paragraph{Marginal concrete-score projection.}
Computing the CCS after each sink removal requires the concrete scores of the remaining variables' marginal distribution. These scores cannot generally be obtained by simply restricting the joint concrete score to the remaining variables. The following identity estimates them via conditional-mean regression.

\begin{proposition}[Marginalization of the concrete score]
\label{prop:marginal-concrete-score}
Let $j\in R\subseteq V$ with count-valued $X_j$. For each neighboring replacement $y\in\mathcal X_j$, the concrete-score entries satisfy
\[
s^R_{j,y}(x_R)
=
\E\!\left[s^V_{j,y}(X_V)\mid X_R=x_R\right].
\]
\end{proposition}

Applying Proposition~\ref{prop:marginal-concrete-score} to the noise-corrupted distribution gives the same conditional-mean identity at each noise level $\sigma$. Using predictions from $s_\theta(x;\sigma)$ as targets, we learn a single mask-conditioned projection network with remaining set $R$ as an input. The resulting estimator $\widehat s_{\theta,\phi}^{\,R}(x_R;\sigma)$, with projection parameters $\phi$, provides marginal concrete scores for different sets $R$ without refitting. Appendices~\ref{app:score-estimation-proofs} and~\ref{app:algorithm-details} provide the theoretical analysis and implementation details, respectively.

\subsection{CCS and OCS Computation}
\label{sec:curvature-score-computation}
We compute the CCS and OCS at near-zero noise levels on held-out folds, using estimators fitted on the other folds. Fold indices on fitted quantities are suppressed.

Let $\widehat\ell_{j,R}(x_R;\sigma)$ denote the projection network's predicted log concrete-score entry for increasing $x_j$ by one. Where all required neighbors lie in the grid, it equals $\log\bigl[\widehat s_{\theta,\phi}^{\,R}(x_R;\sigma)\bigr]_{j,\,x_j+1}$, and for $i,j\in R$ we estimate curvature by
\begin{equation}
\label{eq:estimated-curvatures}
\widehat H_{ij,R}(x_R;\sigma)
:=\widehat\ell_{j,R}(x_R+e_i;\sigma)
-\widehat\ell_{j,R}(x_R;\sigma).
\end{equation}
For the CCS, we first average diagonal curvature estimates over a finite set of noise levels $\mathcal G_\sigma$, without taking absolute values: $\widehat H_{j,R}:=|\mathcal G_\sigma|^{-1} \sum_{\sigma\in\mathcal G_\sigma}\widehat H_{jj,R}(\cdot;\sigma)$.

To estimate the CCS, we first restrict evaluation to observations where the required finite differences are defined. We group these observations by the rank value of $X_j$ and take a sample-size-weighted average of the within-group sample variances of $\widehat H_{j,R}$. For these observations, let $\mathcal J_{j,v}=\{a:X_{aj}=v\}$ and $n_v=|\mathcal J_{j,v}|$. Keeping the levels $\mathcal V_j=\{v:n_v\ge n_{\min}\}$ gives
\begin{equation}
\label{eq:cond-var-score}
\widehat S_j(R)
=
\frac{1}{\sum_{v\in\mathcal V_j}n_v}
\sum_{v\in\mathcal V_j}n_v\,
\widehat{\Var}\!\left(
\{\widehat H_{j,R}(X_a):a\in\mathcal J_{j,v}\}
\right).
\end{equation}
Here $\widehat{\Var}$ is the sample variance.

We compute the OCS separately on each held-out fold. Given an estimated causal order $\widehat\pi$, define $\mathrm{Pred}_{\widehat\pi}(i)=\{j:\widehat\pi(j)<\widehat\pi(i)\}$ and $\widehat R_i=\{i\}\cup\mathrm{Pred}_{\widehat\pi}(i)$. Let $\mathcal I_1,\ldots,\mathcal I_B$ denote these folds. For candidate edge $j\to i$, let $\mathcal I_b^{ij}$ index the evaluation observations used for that edge. For nonempty $\mathcal I_b^{ij}$, we take absolute values before averaging over noise levels and observations:
\begin{equation}
\label{eq:empirical-ocs}
\widehat{\mathrm{OCS}}_{j\to i}^{(b)}
=
\frac{1}{|\mathcal I_b^{ij}|\,|\mathcal G_\sigma|}
\sum_{a\in\mathcal I_b^{ij}}\sum_{\sigma\in\mathcal G_\sigma}
\bigl|\widehat H_{ij,\widehat R_i}(X_{a,\widehat R_i};\sigma)\bigr|.
\end{equation}
\subsection{DAG Recovery}
\label{sec:dag-recovery}

Starting from $R=V$, \DISCO\ repeatedly removes a minimizer of $\widehat S_j(R)$ and assigns that node the last available position. Given this order, parents can be selected using CAM, PCM-GAM, or OCS. Algorithm~\ref{alg:csm-recursive} presents the OCS-based implementation.

For scalable parent selection, we use the OCS estimates without fitting additional regressions or performing conditional-independence tests. Estimated absolute curvatures can be positive even for nonparents. To account for this noise baseline, we standardize the OCS estimates for each target node and fold, obtaining $Z_{j\to i}^{(b)}$. We retain edges whose selection frequency $\Pi_{j\to i}(\tau_Z):=B^{-1}\sum_{b=1}^B\mathbf 1\{Z_{j\to i}^{(b)}>\tau_Z\}$ is at least $\pi_{\min}$.

\begin{algorithm}[H]
\caption{\DISCO\ with OCS-based parent selection.}
\label{alg:csm-recursive}
\begin{algorithmic}
\State \textbf{Input:} marginal concrete-score estimator $\widehat s_{\theta,\phi}^{\,R}$ and evaluation folds $\{\mathcal I_b\}_{b=1}^B$
\end{algorithmic}
\smallskip
\noindent
\begin{minipage}[t]{0.47\linewidth}
\textbf{Ordering}
\begin{algorithmic}
\State $R\gets V$
\While{$|R|>1$}
    \State Compute $\widehat S_j(R)$ for all $j\in R$
    \State $\widehat j_R\in\arg\min_{j\in R}\widehat S_j(R)$
    \State $\widehat\pi(\widehat j_R)\gets|R|$
    \State $R\gets R\setminus\{\widehat j_R\}$
\EndWhile
\State Assign the sole remaining node to position $1$
\end{algorithmic}
\end{minipage}\hfill
\begin{minipage}[t]{0.51\linewidth}
\textbf{Parent selection given $\widehat\pi$}
\begin{algorithmic}
\For{each node $i\in V$}
    \State $\widehat R_i\gets\{i\}\cup\mathrm{Pred}_{\widehat\pi}(i)$
    \State Compute the foldwise $\widehat{\mathrm{OCS}}_{j\to i}^{(b)}$
    \State Standardize and form $\Pi_{j\to i}(\tau_Z)$
    \State Select $j\to i$ if $\Pi_{j\to i}(\tau_Z)\ge\pi_{\min}$
\EndFor
\State \textbf{Output:} DAG $\widehat G$
\end{algorithmic}
\end{minipage}
\end{algorithm}

For the consistency result, we consider a variant of Algorithm~\ref{alg:csm-recursive} that selects parents by thresholding the raw OCS at a deterministic level $\lambda_n>0$.

\begin{theorem}[Plug-in Consistency of \DISCO{} with raw OCS thresholding]
\label{thm:plugin-consistency}
Under the assumptions of Theorem~\ref{thm:sink-identification} and Assumption~\ref{ass:unified-curvature-consistency}, this variant recovers a causal order of $G$ with probability tending to one as $n\to\infty$. If $\lambda_n\to0$ and the curvature-estimation error $\varepsilon_n$ from Assumption~\ref{ass:unified-curvature-consistency} satisfies $\varepsilon_n=o_p(\lambda_n)$, then $\Pr\{\widehat G=G\}\longrightarrow1$.
\end{theorem}

\paragraph{Transformation invariance.}
\label{sec:support-index}
For the population result on count data, let $T_j$ map each value in the ordered support of coordinate $j$ to its support rank. The fitted map $\widehat T_j$ is its sample counterpart.

\begin{proposition}[Invariance to strictly increasing transformations]
\label{prop:support-index-invariance}
\textup{\DISCO} with OCS-based parent selection returns the same graph under strictly increasing transformations of each variable when using the same seed.
\end{proposition}

Proofs of the formal results in this section are provided in Appendix~\ref{app:disco-proofs}.

\paragraph{Computational complexity.}
\label{sec:disco-cost}
Following DiffAN~\citep{sanchez2023diffan}, we fix the training epochs and evaluation sample size, and count each score evaluation at unit cost. With fixed hidden widths and mask draws per observation, the rank transform and training take $O(nd)$ for $n$ training samples on consecutive counts. Ordering requires $O(d^2)$ shifted-state evaluations in $O(d)$ sequential batches. OCS-based parent selection uses $O(d)$ batches and $O(d^2)$ candidate-edge operations. These counts give $O(nd+d^2)$ for the OCS-based implementation; CAM and PCM-GAM require additional regressions or tests.

\section{Numerical Experiments}
\label{sec:experiments}

We evaluate whether conditional curvature restores the ordering signal when nonconstant own-factor curvature invalidates the constant-curvature criterion, whether \DISCO{} achieves family-agnostic count-DAG recovery within the semiparametric GLM DAG class, and whether it scales to large DAGs. We compare \DISCO{} with DiffAN \citep{sanchez2023diffan}, \NOTEARSls{} \citep{zheng2020learning}, ODS \citep{park2018qvf}, and MRS \citep{park2019ghd}. ODS (Oracle) receives the true nodewise conditional families and their fixed parameters, whereas \DISCO{} receives no family labels. For parent selection, we compare CAM's additive-regression pruning \citep{buhlmann2014cam} and PCM-GAM's conditional-mean independence testing \citep{lundborg2024projected}. Unless otherwise stated, \DISCO{} uses OCS-based parent selection. Experimental settings and evaluation metrics are given in Appendix~\ref{app:experimental-setup}. Additional experiments and algorithm diagnostics appear in Appendices~\ref{app:disco-exp} and~\ref{app:estimator-diagnostics}.

\paragraph{DAG recovery versus sample size.}
Figure~\ref{fig:curvature-scaling} compares \DISCO{} with the baselines on DAGs with $d=100$ as $n\in\{500,1000,2000,3000,4000,5000\}$ increases: NB with $r=1$ in panel (a), NB with $r=6$ in panel (b), and the mixed NB $\to$ Bin $\to$ Poi configuration in panel (c). \DISCO{} attains the highest median $F_1$ at larger sample sizes under $r=6$ and mixed families, whereas ODS (Oracle) remains strongest under $r=1$. As illustrated by the negative binomial example in Section~\ref{sec:bivariate}, own-factor curvature vanishes when $r=1$, so the constant-curvature criterion is already sufficient and estimating the more general conditional curvature criterion can incur an additional finite-sample efficiency cost. Nevertheless, \DISCO{} shows improved recovery at larger sample sizes. DiffAN shows no consistent gains with more data under $r=6$ and mixed families, where the constant-curvature criterion fails because sink curvature is nonconstant.

\begin{figure}[htbp]
\centering
\begin{minipage}{\linewidth}
\centering
\includegraphics[width=\linewidth]{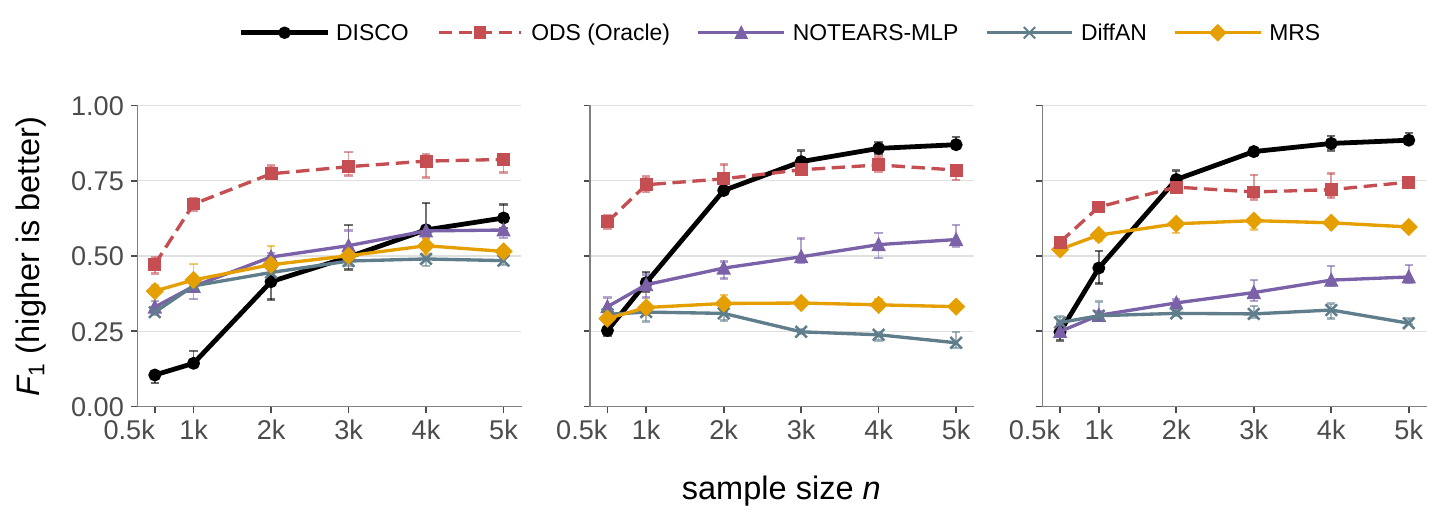}
\par\vspace{2pt}
{\normalsize
\makebox[\linewidth][l]{%
\hspace*{0.228411\linewidth}%
\makebox[0pt][c]{(a) NB ($r=1$)}%
\hspace*{0.314054\linewidth}%
\makebox[0pt][c]{(b) NB ($r=6$)}%
\hspace*{0.314054\linewidth}%
\makebox[0pt][c]{(c) NB $\to$ Bin $\to$ Poi}%
}}
\end{minipage}
\caption{End-to-end $F_1$ versus sample size on DAGs with $d=100$: (a) NB with $r=1$; (b) NB with $r=6$; and (c) the mixed NB $\to$ Bin $\to$ Poi configuration. Points and error bars show the median and interquartile range. ODS (Oracle) uses the true nodewise conditional families and their fixed parameters.}
\label{fig:curvature-scaling}
\end{figure}

\paragraph{Conditional curvature ablation.}
We compare CCS and the constant-curvature criterion on the same sets of remaining nodes using the same fitted networks ($d=50$, $n=5000$). Table~\ref{tab:ccs-ablation}(a) shows that CCS improves the accuracy of identifying sinks, with larger gains for Poisson and negative binomial data than for binomial data. This improvement does not always increase DAG recovery accuracy: for binomial data, mean $F_1$ is $0.877$ with CCS and $0.884$ with the constant-curvature criterion. Table~\ref{tab:ccs-full-dag} in Appendix~\ref{app:ccs-diagnostics} reports DAG recovery for all three conditional families.

\paragraph{Parent recovery given an estimated order.}
We compare OCS, CAM, and PCM-GAM given the same estimated \DISCO{} order at $d=50$ and $n=5000$, using disjoint ordering and parent-recovery samples. CAM improves mean $F_1$ over OCS for NB and Bin, whereas OCS performs best for Poi (Table~\ref{tab:parent-methods-main}(b)). Appendix~\ref{app:diagnostic-parent-selection} gives the test settings, SHD, and true-order comparisons.

\begin{table}[H]
\TableStyle
\caption{Ordering and parent-selection comparisons ($d=50$, $n=5000$). (a) Sink-selection accuracy: Constant denotes the constant-curvature criterion; gain is in percentage points. (b) Mean $F_1$ given the same estimated \DISCO{} order.}
\label{tab:ccs-ablation}
\label{tab:parent-methods-main}
\begin{minipage}[t]{0.47\linewidth}
\centering
(a) Sink selection (\%)\par\smallskip
\small
\begin{tabular*}{\linewidth}{@{\extracolsep{\fill}}lrrr@{}}
\toprule
Family & Constant & CCS & Gain \\
\midrule
Poi & 68.3 & \textbf{81.7} & $+13.3$ \\
NB & 73.3 & \textbf{93.3} & $+20.0$ \\
Bin & 80.0 & \textbf{83.3} & $+3.3$ \\
\bottomrule
\end{tabular*}
\end{minipage}\hfill
\begin{minipage}[t]{0.49\linewidth}
\centering
(b) Parent recovery ($F_1$)\par\smallskip
\small
\begin{tabular*}{\linewidth}{@{\extracolsep{\fill}}lrrr@{}}
\toprule
Selector & Poi & NB & Bin \\
\midrule
OCS & \textbf{0.858} & 0.751 & 0.828 \\
CAM & 0.850 & \textbf{0.823} & \textbf{0.898} \\
PCM-GAM & 0.712 & 0.703 & 0.738 \\
\bottomrule
\end{tabular*}
\end{minipage}
\end{table}

\paragraph{Scalability with graph size.}
Table~\ref{tab:scalability-main} reports $F_1$ and total wall-clock time as $d$ grows on Poisson DAGs with $n=5000$. At $d=1000$, \DISCO{} attains median $F_1=0.769$ in 57.4 minutes on a single GPU, whereas all baselines exceed the 240-minute time limit. This demonstrates practical scalability; the runtimes are not intended as a hardware-matched speed comparison.

\begingroup
\begin{table}[htbp]
\TableStyle
\caption{Scalability on Poisson DAGs ($n=5000$; numerical entries are medians). \emph{Timeout}: runtime exceeds 240 minutes.}
\label{tab:scalability-main}
\begin{minipage}{\linewidth}
\normalsize\rmfamily
\setlength{\parskip}{0pt}
\setlength{\tabcolsep}{0pt}
\renewcommand{\arraystretch}{1.0}
\setlength{\aboverulesep}{0.3ex}
\setlength{\belowrulesep}{0.4ex}
\begin{tabular*}{\linewidth}{@{\extracolsep{\fill}}lcc@{\hspace{5pt}}cc@{\hspace{5pt}}cc@{\hspace{5pt}}cc@{}}
\toprule
& \multicolumn{2}{c}{$d=100$}
& \multicolumn{2}{c}{$d=200$}
& \multicolumn{2}{c}{$d=500$}
& \multicolumn{2}{c}{$d=1000$} \\
\cmidrule(lr){2-3}\cmidrule(lr){4-5}\cmidrule(lr){6-7}\cmidrule(lr){8-9}
\makebox[0.20\linewidth][l]{Method}
& \makebox[0.08\linewidth]{$F_1$} & \makebox[0.105\linewidth]{\shortstack{Total time\\(min)}}
& \makebox[0.08\linewidth]{$F_1$} & \makebox[0.105\linewidth]{\shortstack{Total time\\(min)}}
& \makebox[0.08\linewidth]{$F_1$} & \makebox[0.105\linewidth]{\shortstack{Total time\\(min)}}
& \makebox[0.08\linewidth]{$F_1$} & \makebox[0.105\linewidth]{\shortstack{Total time\\(min)}} \\
\midrule
\textbf{\DISCO{}} & \textbf{0.915} & \makebox[2.7em][r]{2.7} & \textbf{0.835} & \makebox[2.7em][r]{3.5} & \textbf{0.769} & \makebox[2.7em][r]{10.8} & \textbf{0.769} & \makebox[2.7em][r]{57.4} \\
\textsc{ODS} (Oracle) & 0.725 & \makebox[2.7em][r]{12.5} & 0.713 & \makebox[2.7em][r]{37.0} & 0.729 & \makebox[2.7em][r]{204.7} & -- & \emph{Timeout} \\
MRS & 0.786 & \makebox[2.7em][r]{3.5} & 0.763 & \makebox[2.7em][r]{10.8} & 0.742 & \makebox[2.7em][r]{96.3} & -- & \emph{Timeout} \\
\NOTEARSls{} & 0.467 & \makebox[2.7em][r]{174.1} & -- & \emph{Timeout} & -- & \emph{Timeout} & -- & \emph{Timeout} \\
DiffAN & -- & \emph{Timeout} & -- & \emph{Timeout} & -- & \emph{Timeout} & -- & \emph{Timeout} \\
\bottomrule
\end{tabular*}\par
\end{minipage}
\end{table}
\endgroup

\paragraph{Family-agnostic DAG learning on real MLB data.}
We revisit the Lahman batting data previously analyzed for count-DAG learning by \citet{park2019mrs}, using $d=17$ count variables from $n=1{,}400$ player--seasons during 2019--2025. All methods in this diagnostic use the same four-level quantile representation (q4), chosen to increase observations per level for CCS estimation. We evaluate ordering against seven domain-informed accounting and containment relations among standard baseball statistics \citep{mlbRecording}. ODS requires a conditional-family specification, and its mean ordering agreement \(A_{\mathrm{top}}\) with these references varies from $0.333$ to $0.429$ across family choices. In contrast, \DISCO{} requires no family or link input and attains mean $A_{\mathrm{top}}=0.905$. Although the partial references do not define a ground-truth DAG, \DISCO{} shows higher agreement with these domain-informed references than the three fixed-family ODS variants, without requiring a family choice. Appendix~\ref{app:lahman} provides the complete diagnostics, including pairwise family-specification comparisons.

\section{Conclusion}
\label{sec:conclusion}

\DISCO{} enables score-matching-based causal discovery from count data without prespecifying nodewise conditional families, relaxing a central assumption of existing identifiable count-DAG methods. The CCS and OCS provide ordering and parent-recovery criteria for nonlinear semiparametric GLM DAGs under the stated assumptions. A joint concrete-score network and marginal projection allow \DISCO{} to estimate both criteria, with OCS-based recovery scaling to 1,000-node DAGs in our experiments. The main remaining theoretical challenge is to establish end-to-end consistency of \DISCO{}, from finite-noise discrete-diffusion score estimation and marginal projection to curvature estimation and DAG recovery.

\subsection*{AI Use Statement}
Generative AI tools were used to assist with aspects of the proofs and experimental work. The authors reviewed the AI-assisted material and take responsibility for the final manuscript and results.

\subsection*{Reproducibility statement}
Implementation details, including preprocessing, training objectives, mask sampling, and parent selection, are provided in Appendix~\ref{app:algorithm-details}. Appendix~\ref{app:disco-exp} specifies the data-generating mechanisms, baseline configurations, evaluation metrics, and replication protocols. Appendix~\ref{app:estimator-diagnostics} describes the algorithm diagnostics, and Appendix~\ref{app:lahman} gives the Lahman cohort construction, paired cross-validation design, and reference sets. The appendix also contains the assumptions and proofs of the theoretical results.

\appendix
\addtocontents{toc}{\protect\setcounter{tocdepth}{2}}
\AppendixContents

\section{Notation}
\label{app:notation}

\begin{table}[H]
\TableStyle
\caption{Recurring symbols.}
\label{tab:notation}
\begin{tabular}{@{}p{0.28\linewidth}p{0.68\linewidth}@{}}
\toprule
Symbol & Meaning (defined in) \\
\midrule
$G=(V,E)$, $|V|=d$ & DAG and number of nodes (Sec.~\ref{sec:prelim}) \\
$\widehat G$ & estimated DAG (Sec.~\ref{sec:dag-recovery}) \\
$\pa_G(j),\ \ch_G(j)$ & parents and children of $j$ (a sink has $\ch_G(j)=\emptyset$) \\
source node & a node $j$ with $\pa_G(j)=\emptyset$ \\
$R\subseteq V$ & remaining node set, ancestral when obtained by recursive sink removal \\
$G_R,\ P_R,\ p_R$ & induced subgraph, marginal distribution of $X_R$, and its density \\
$p_{R,\sigma}$ & density of the noise-corrupted $R$-marginal (App.~\ref{app:score-estimation-proofs}) \\
$\pi,\ \mathrm{Pred}(i),\ R_i$ & causal order, predecessors $\{j:\pi(j)<\pi(i)\}$, and $R_i=\{i\}\cup\mathrm{Pred}(i)$ (Sec.~\ref{sec:ccs}) \\
$\widehat\pi,\ \mathrm{Pred}_{\widehat\pi}(i),\ \widehat R_i$ & estimated counterparts of the preceding row (Sec.~\ref{sec:curvature-score-computation}) \\
$x_A^{j\to y}$ & $x_A$ with coordinate $j$ reset to $y$ (Sec.~\ref{sec:prelim}) \\
$\Dop_j,\ \Dop_j^2$ & unit coordinate (second) difference, with $\partial_j,\partial_j^2$ if continuous \\
$b_j$ & unspecified base function of a non-source conditional (Def.~\ref{def:ef-dag}) \\
$g_j,\ f_j$ & link and its corresponding index \\
$\eta_j$ & canonical-parameter function encoding parent dependence \\
$A_j$ & log-partition of a non-source conditional's natural exponential family \\
$\kappa_j(x_j)$ & own-factor curvature of node $j$ (Sec.~\ref{sec:multivariate}) \\
$\psi_X,\ \psi_{j,R}$ & residual curvature in the bivariate or marginal decomposition (Secs.~\ref{sec:bivariate}--\ref{sec:multivariate}) \\
$H_{j,R}=\Dop_j^2\log p_R$ & diagonal curvature of the marginal (Sec.~\ref{sec:ccs}) \\
$H_{ij,R}=\Dop_i\Dop_j\log p_R$ & off-diagonal curvature (parent recovery) \\
$\mathcal A_j^{(2)},\ \mathcal A_{ij}^{(1,1)}$ & admissible-difference events for $\Dop_j^2$, $\Dop_i\Dop_j$ \\
$S_j(R)$ & conditional curvature score (CCS), \eqref{eq:ccs-def} \\
$\mathrm{OCS}_{j\to i}$ & off-diagonal curvature score, \eqref{eq:parent-score} \\
$O_{j\to i}(R)$ & OCS on a general admissible marginal in the consistency proof \\
$T_j$ & population rank transform (Sec.~\ref{sec:support-index}) \\
$\widehat T_j$ & fitted rank transform from the score-training split (Sec.~\ref{sec:rank-preprocess}) \\
$s_{j,y}^R$ & concrete-score entry, neighboring when $y=x_j\pm1$ \\
$s_\theta$ & joint concrete-score network \\
$\widehat s_{\theta,\phi}^R$ & $R$-marginal concrete-score estimate \\
$\widehat\ell_{j,R}(\cdot;\sigma)$ & predicted log concrete-score entry for increasing $x_j$ by one (Sec.~\ref{sec:curvature-score-computation}) \\
$\widehat S_j(R),\ \widehat{\mathrm{OCS}}_{j\to i}^{(b)}$ & estimated CCS and foldwise estimated OCS (Sec.~\ref{sec:curvature-score-computation}) \\
\bottomrule
\end{tabular}
\end{table}

\paragraph{Admissible domains.}
A finite difference is defined only where all required shifted values remain in the support; we call this its \emph{admissible domain}. For $r\in\{1,2\}$, define
\[
\mathcal X_j^{\circ,r}
=
\begin{cases}
\{x_j\in\mathcal X_j:x_j,x_j+1,\ldots,x_j+r\in\mathcal X_j\},
& X_j\text{ count},\\
\mathring{\mathcal X}_j,
& X_j\text{ continuous}.
\end{cases}
\]
Thus $\Dop_j^r$ requires $x_j\in\mathcal X_j^{\circ,r}$, and, for $i\ne j$, $\Dop_i\Dop_j$ requires $x_i\in\mathcal X_i^{\circ,1}$ and $x_j\in\mathcal X_j^{\circ,1}$. The corresponding curvature events are
\[
\mathcal A_j^{(2)}=\{X_j\in\mathcal X_j^{\circ,2}\},
\qquad
\mathcal A_{ij}^{(1,1)}
=\{X_i\in\mathcal X_i^{\circ,1},\ X_j\in\mathcal X_j^{\circ,1}\},
\quad i\ne j.
\]
Continuous coordinates impose only the support-interior restriction, which has probability one.

For these curvatures, population moments and almost-sure statements use the marginal distribution conditioned on the corresponding event. For the CCS, both the outer expectation and the conditional variance use $P_R(\,\cdot\mid\mathcal A_j^{(2)})$; the OCS uses $P_{R_i}(\,\cdot\mid\mathcal A_{ij}^{(1,1)})$. If the event is determined by the conditioning variables, this restriction leaves the conditional distribution of the other variables unchanged.

\section{Regularity Conditions}
\label{app:standing}

\subsection{Common Regularity}
\label{app:distribution-regularity}

We assume finite, strictly positive densities on product-support interiors. Each coordinate support is a connected interval with nonempty interior or a consecutive integer set with at least three levels. For the distribution-level characterization, $\log p$ is jointly $C^2$ in the continuous coordinates for each fixed choice of count coordinates.

\subsection{Semiparametric GLM Representations}
\label{app:glm-regularity}

For the semiparametric GLM representations in Definition~\ref{def:ef-dag}, count supports are additionally bounded below. Source densities and non-source base functions are finite and strictly positive on their support interiors, and their logarithms are $C^2$ in continuous coordinates. Each $\eta_j$ is finite and jointly $C^2$ in its continuous parent coordinates for each fixed choice of count parents. Required finite differences are finite on their admissible domains, and the canonical parameters lie in the interior of the canonical-parameter domain:
\[
\eta_j(\mathring{\mathcal X}_{\pa_G(j)})
\subseteq\operatorname{int}\Xi_j,
\qquad
\Xi_j:=\{t\in\mathbb R:A_j(t)<\infty\}.
\]
This interiority justifies differentiation under the integral, so $A_j$ is smooth and $A_j''(t)=\Var_t(X_j)\in(0,\infty)$ on $\operatorname{int}\Xi_j$.

\begin{regularity}[Affine reversibility]
\label{cond:affine-reverse}
In the affine branch of Theorem~\ref{thm:link-index}, write $\eta(x)=\beta_0+\beta_1x$ with $\beta_1\neq0$, and define
\[
\widetilde b_X(x)=p_X(x)e^{-A_Y(\beta_0+\beta_1x)},
\qquad
\widetilde\Xi_X
=\left\{t:\int_{\mathcal X_X}\widetilde b_X(x)e^{tx}\,d\nu_X(x)<\infty\right\}.
\]
For count $Y$, require $\beta_1e\in\operatorname{int}\widetilde\Xi_X$ at each support endpoint $e$. There are two endpoints for finite support and one for one-sided infinite support. No additional restriction is imposed for continuous $Y$.
\end{regularity}

Integrating out a sink removes its conditional density from the DAG factorization. Therefore an ancestral marginal is finite and strictly positive and retains the remaining conditional densities and parent sets. The nonlinearity condition is also preserved when present. The own-factor curvature is $\kappa_j(x_j)=\Dop_j^2\log b_j(x_j)$ for a non-source and $\kappa_j(x_j)=\Dop_j^2\log p_j(x_j)$ for a source, consistent with Section~\ref{sec:multivariate}. Its curvature decomposition is
\begin{equation}
\label{eq:nonsink-decomp}
\begin{aligned}
H_{j,R}&=\kappa_j(x_j)+\psi_{j,R}(x_R),\\
\psi_{j,R}
&=\sum_{c\in\ch_{G_R}(j)}
\Dop_j^2\!\left\{
x_c\eta_c(x_{\pa_{G_R}(c)})
-A_c\bigl(\eta_c(x_{\pa_{G_R}(c)})\bigr)
\right\}.
\end{aligned}
\end{equation}
The child base functions do not depend on $x_j$ and disappear under $\Dop_j^2$.

\subsection{Moment Conditions}
\label{app:score-regularity}

\paragraph{Score moments.}
Where the CCS is used, we assume $H_{j,R}\in L^2(P_R(\cdot\mid\mathcal A_j^{(2)}))$ for every evaluated ancestral set $R$ and $j\in R$. In the bivariate setting, this means $H_Z\in L^2(P(\cdot\mid\mathcal A_Z^{(2)}))$ for $Z\in\{X,Y\}$. Where the OCS is used, we assume $H_{ij,R_i}\in L^1(P_{R_i}(\cdot\mid\mathcal A_{ij}^{(1,1)}))$ for every evaluated predecessor marginal $R_i$ and $j\in\mathrm{Pred}(i)$.

\paragraph{Score-entropy regression.}
For each positive active target $U$, assume $\E[U\log^+U]<\infty$, where $\log^+u=\max\{\log u,0\}$. This moment condition holds on finite population supports, but a finite empirical rank grid alone does not ensure it.

\section{Proofs for Section~\ref{sec:identifiability}}
\label{app:identifiability-proofs}

All derivatives, finite differences, and conditional moments below use the notation and admissible-event convention in Appendix~\ref{app:notation}.

\subsection{Proof of Theorem~\ref{thm:bivariate-sink}}
\label{app:bivariate-proofs}

\begin{proof}
Write $F(x,y)=\log p(x,y)$. By Definition~\ref{def:ccs}, $S_Y=0$ exactly when $\Dop_Y^2F(X,Y)$ is almost surely a function of $Y$ alone. Support positivity and smoothness extend this equality to the admissible domain, so $\Dop_Y^2F(x,y)$ is independent of $x$.

Fix $x_0\in\mathring{\mathcal X}_X$ and let $r_x(y)=F(x,y)-F(x_0,y)$. Since $\Dop_Y^2r_x=0$, a connected continuous support or a consecutive count support implies that $r_x$ is affine. Thus there are functions $a(x)$ and $\eta(x)$ such that $r_x(y)=a(x)+y\eta(x)$. Define $b_Y(y):=\exp\{F(x_0,y)\}$. Then
\begin{equation}
\label{eq:ccs-bivariate-factorization}
\log p(x,y)=a(x)+\log b_Y(y)+y\eta(x).
\end{equation}
For any fixed distinct $y_0,y_1$ in the support interior, $\eta(x)=\{r_x(y_1)-r_x(y_0)\}/(y_1-y_0)$. Thus $a$ and $\eta$ are finite, and they are $C^2$ when $X$ is continuous.

Define the log-partition function $A_Y(t)=\log\int_{\mathcal X_Y}b_Y(u)e^{ut}\,d\nu_Y(u)$. Exponentiating \eqref{eq:ccs-bivariate-factorization} and integrating over $y$ gives
\[
p_X(x)
=e^{a(x)}\int_{\mathcal X_Y}b_Y(u)e^{u\eta(x)}\,d\nu_Y(u)
=\exp\{a(x)+A_Y(\eta(x))\}.
\]
The marginal density $p_X(x)$ is finite and positive for $P_X$-almost every $x$. For each such $x$, the integral is therefore finite, and division by $p_X(x)$ yields
\[
p(y\mid x)
=\frac{e^{a(x)}b_Y(y)e^{y\eta(x)}}{e^{a(x)+A_Y(\eta(x))}}
=b_Y(y)\exp\{y\eta(x)-A_Y(\eta(x))\}.
\]
This gives \eqref{eq:ef-density}. Finiteness of $A_Y(\eta(x))$ alone does not imply that $\eta(x)$ lies in the interior of the canonical-parameter domain.

Conversely, suppose the conditional density has this form. In $\log p(x,y)=\log p_X(x)+\log b_Y(y)+y\eta(x)-A_Y(\eta(x))$, all terms other than $\log b_Y(y)$ are constant or affine in $y$. Hence $\Dop_Y^2\log p(x,y)=\Dop_Y^2\log b_Y(y)$, which depends only on $y$, and $S_Y=0$.
\end{proof}

\subsection{Proof of Theorem~\ref{thm:link-index}}
\label{app:bivariate-reverse-proof}

\begin{proof}
In the given $X\to Y$ representation, $\eta$ is nonconstant and $A_Y''>0$, so the conditional mean $\E[Y\mid X=x]=A_Y'(\eta(x))$ is nonconstant. Thus $X$ and $Y$ are dependent, leaving only the two edge directions to compare.

First suppose $\eta(x)=\beta_0+\beta_1x$, where $\beta_1\neq0$. Use $\widetilde b_X$ and $\widetilde\Xi_X$ from Condition~\ref{cond:affine-reverse}, and write $\widetilde A_X(t)=\log\int_{\mathcal X_X}\widetilde b_X(x)e^{tx}\,d\nu_X(x)$. The joint density can be written as $p(x,y)=b_Y(y)e^{\beta_0y}\widetilde b_X(x)e^{x\beta_1y}$. Integrating over $x$ gives $p_Y(y)=b_Y(y)\exp\{\beta_0y+\widetilde A_X(\beta_1y)\}$. Dividing the joint density by this marginal density therefore gives $p(x\mid y) =\widetilde b_X(x)\exp\{x\beta_1y-\widetilde A_X(\beta_1y)\}$.

To show that the constructed $Y\to X$ representation belongs to the same model class, we check its regularity. First consider canonical-parameter interiority. If $Y$ is continuous, the integral defining $p_Y$ is finite for almost every $y$. Around any interior support point $y$, choose one such finite point on each side. H\"older's inequality makes the finiteness domain $\widetilde\Xi_X$ convex, so $\beta_1y$ lies in its interior. For count $Y$, the integral is finite at every support value because $p_Y(y)\leq1$. Every level with neighbors on both sides therefore gives an interior parameter. Condition~\ref{cond:affine-reverse} covers the endpoints.

The reverse parameter $\widetilde\eta(y)=\beta_1y$ is nonconstant because $\beta_1\neq0$. To check smoothness, observe that
\[
\begin{aligned}
\log\widetilde b_X(x)&=\log p_X(x)-A_Y(\beta_0+\beta_1x),\\
\log p_Y(y)&=\log b_Y(y)+\beta_0y+\widetilde A_X(\beta_1y).
\end{aligned}
\]
Canonical-parameter interiority and the stated smoothness make these expressions $C^2$ in continuous variables and finite at count values. Finally, its affine-reversibility base is $p_Y(y)e^{-\widetilde A_X(\beta_1y)}=b_Y(y)e^{\beta_0y}$, with parameter domain $\Xi_Y-\beta_0$. If $X$ is count, then at each support endpoint $e$, forward canonical-parameter interiority gives $\beta_0+\beta_1e\in\operatorname{int}\Xi_Y$, or equivalently $\beta_1e\in\operatorname{int}(\Xi_Y-\beta_0)$. Thus the reverse representation also satisfies Condition~\ref{cond:affine-reverse}. Taking $P_Y$ as the source therefore gives a reverse representation in the same model class. The direction is consequently not identifiable when $\eta$ is affine.

Conversely, suppose a reverse semiparametric GLM representation exists, with canonical-parameter function $\widetilde\eta$. Choose distinct $y_0,y_1$ for which both conditional representations hold. Positivity permits any two support values for count $Y$. When $Y$ is continuous, Fubini's theorem gives a full-measure set of valid values, from which two distinct points can be chosen. The forward representation and Bayes' rule give
\[
\log p(x\mid y_1)-\log p(x\mid y_0)
=(y_1-y_0)\eta(x)+C,
\]
where $C$ is independent of $x$. The reverse representation expresses the same difference as $x\{\widetilde\eta(y_1)-\widetilde\eta(y_0)\}+C'$, where $C'$ is also independent of $x$. Equating the two expressions and dividing by $y_1-y_0\neq0$ shows that $\eta(x)$ is affine almost everywhere. Support positivity and smoothness extend this equality to the support interior. Therefore a reverse representation exists exactly when $\eta$ is affine, proving the result.
\end{proof}

\subsection{Proof of Theorem~\ref{thm:sink-identification}}
\label{app:multivariate-proofs}

\begin{proof}
\textit{(I) Sink characterization.} Fix an ancestral remaining set $R\subseteq V$ and a node $j\in R$. Write $P_R^{(j)}=P_R(\,\cdot\mid\mathcal A_j^{(2)})$, which equals $P_R$ when $X_j$ is continuous. All probabilities, expectations, and variances in this part use $P_R^{(j)}$, with the convention in Appendix~\ref{app:notation}. We first establish
\begin{equation}
\label{eq:structural-sink}
j\text{ is a sink of }G_R
\quad\Longleftrightarrow\quad
H_{j,R}(X_R)=h_{j,R}(X_j)
\quad P_R^{(j)}\text{-a.s. for some measurable }h_{j,R}.
\end{equation}

If $j$ is a sink of $G_R$, \eqref{eq:nonsink-decomp} gives $H_{j,R}=\kappa_j(X_j)$, proving the forward implication in \eqref{eq:structural-sink} with $h_{j,R}=\kappa_j$.

For the reverse implication, we prove the contrapositive. Suppose $j$ is a nonsink, and choose a sink $c^\star$ in the subgraph induced by $\ch_{G_R}(j)$. Such a node exists because this subgraph is a nonempty finite DAG. Then $c^\star$ is not a parent of any other child of $j$, so only its own child contribution in \eqref{eq:nonsink-decomp} depends on $X_{c^\star}$. Let $U=X_{R\setminus\{c^\star\}}$. All parents of $c^\star$, including $j$, belong to $R\setminus\{c^\star\}$. Define the function $a(u):=(\Dop_j^2\eta_{c^\star}) \bigl(u_{\pa_{G_R}(c^\star)}\bigr)$. Thus $\eta_{c^\star}(X_{\pa_{G_R}(c^\star)})$ and its log-partition term are functions of $U$ alone. Since $\Dop_j^2$ leaves $X_{c^\star}$ fixed, linearity gives $H_{j,R}=a(U)X_{c^\star}+B(U)$, where $B(U)$ collects all terms independent of $X_{c^\star}$. Definition~\ref{def:nonlinear-ef-dag} gives an admissible $u$ with $a(u)\neq0$. Support positivity and smoothness then give $\mathbb P\{a(U)\neq0\}>0$. Support positivity also makes $X_{c^\star}\mid U$ nondegenerate, including under the admissible-event restriction. Hence $H_{j,R}\mid U$ is nondegenerate on $\{a(U)\neq0\}$. If $H_{j,R}$ were a function of $X_j$ alone, it would be fixed given $U$, since $U$ contains $X_j$. This contradiction proves the reverse implication in \eqref{eq:structural-sink}.

For a nonsink, $H_{j,R}\in L^2$ and the identity $X_{c^\star}=\{H_{j,R}-B(U)\}/a(U)$ give a finite conditional second moment for $X_{c^\star}$ given $U$ on $\{a(U)\neq0\}$. The conditional nondegeneracy above therefore implies $\Var(H_{j,R}\mid U)=a(U)^2\Var(X_{c^\star}\mid U)>0$ on this event. Since $U$ includes $X_j$, the conditional law of total variance gives
\[
0<\E\{\Var(H_{j,R}\mid U)\}
\leq\E\{\Var(H_{j,R}\mid X_j)\}=S_j(R)<\infty.
\]
Thus every nonsink has positive CCS, and every sink has zero CCS.

\medskip
\noindent\textit{(II) Parent recovery.}
Fix a causal order, a node $i$, and a predecessor $j\in\mathrm{Pred}(i)$. Write $R_i=\{i\}\cup\mathrm{Pred}(i)$. Expectations below use the admissible-event convention for $\mathcal A_{ij}^{(1,1)}$.

The predecessor set contains all parents of $i$ and no descendants of $i$. The local Markov property therefore gives $p(x_i\mid x_{\mathrm{Pred}(i)})=p_i(x_i\mid x_{\pa_G(i)})$. Since $\log p_{R_i}=\log p(x_i\mid x_{\mathrm{Pred}(i)})+\log p_{\mathrm{Pred}(i)}$ and the second term is independent of $x_i$, we have $H_{ij,R_i} =\Dop_i\Dop_j\log p_i(x_i\mid x_{\pa_G(i)})$. If $i$ is a source, its conditional density equals its marginal density $p_i(x_i)$, so this off-diagonal curvature is zero for every predecessor, as required.

For a non-source, applying $\Dop_i$ to the log conditional density in \eqref{eq:ef-density} gives $\Dop_i\log b_i(x_i)+\eta_i(x_{\pa_G(i)})$. The mixed operators commute on the admissible domain. Applying $\Dop_j$ therefore removes the first term and yields
\[
H_{ij,R_i}=\Dop_j\eta_i,
\qquad
\mathrm{OCS}_{j\to i}
=\E\bigl[\lvert\Dop_j\eta_i(X_{\pa_G(i)})\rvert\bigr].
\]
If $j$ is not a parent, $\eta_i$ is independent of $x_j$. For a parent, the nonconstant dependence in Definition~\ref{def:ef-dag} gives a point where $\Dop_j\eta_i\neq0$; support positivity and smoothness give positive probability to this event. Hence
\begin{equation}
\label{eq:structural-edge}
j\in\pa_G(i)
\quad\Longleftrightarrow\quad
\mathbb P\!\left(
|H_{ij,R_i}|>0
\,\middle|\,
\mathcal A_{ij}^{(1,1)}
\right)>0.
\end{equation}
Since $H_{ij,R_i}\in L^1$, the OCS is finite and positive exactly for parents. For example, $\eta_i(x_j)=a+bx_j$ with $b\neq0$ gives $\mathrm{OCS}_{j\to i}=|b|>0$, so nonlinearity is not needed in this part.
\end{proof}

\subsection{Proof of Corollary~\ref{cor:dag-identifiability}}

\begin{proof}
Suppose that $P$ has nonlinear semiparametric GLM DAG representations on both $G$ and $\widetilde G$. For any remaining set $R$ that is ancestral for both representations, the right-hand side of \eqref{eq:structural-sink} depends only on $P_R$. It therefore identifies the same sink set in $G_R$ and $\widetilde G_R$. Starting from $R=V$, choose a common sink using a fixed tie-breaking rule and remove it. The new set remains ancestral for both graphs, so induction gives a single ordering $\pi$ compatible with both DAG representations.

For each node $i$, $R_i=\{i\}\cup\mathrm{Pred}(i)$ is ancestral for both representations, so \eqref{eq:structural-edge} applies to both graphs. Its right-hand side depends only on $P_{R_i}$ and hence recovers the same parent set for $i$. Thus $G=\widetilde G$. Both \eqref{eq:structural-sink} and \eqref{eq:structural-edge} were established without moment conditions.
\end{proof}

\section{Additional Identifiability Results}
\label{app:additional-identifiability}

\subsection{Strict Extension of SCORE}
\label{app:score-proofs}

\begin{corollary}[Strict extension of SCORE]
\label{cor:continuous-scms}
SCORE's identifiable nonlinear Gaussian ANMs form a proper subclass of the continuous nonlinear semiparametric GLM DAGs identified by Corollary~\ref{cor:dag-identifiability}. For a nonlinear semiparametric GLM DAG with constant own-factor curvature at every node, including these Gaussian ANMs, every ancestral remaining set $R$ and every $j\in R$ satisfy
\[
\Var(H_{j,R})=0
\quad\Longleftrightarrow\quad
S_j(R)=0
\quad\Longleftrightarrow\quad
j\text{ is a sink of }G_R.
\]
\end{corollary}

\begin{proof}
For every non-source node in a Gaussian ANM, choose
\[
b_i(x)=(2\pi\tau_i^2)^{-1/2}e^{-x^2/(2\tau_i^2)},
\qquad
A_i(t)=\frac{\tau_i^2t^2}{2}.
\]
Then $\eta_i=f_i/\tau_i^2$, so positive scaling preserves the nonlinearity condition and Corollary~\ref{cor:dag-identifiability} applies under the regularity conditions in Appendix~\ref{app:glm-regularity}. Gaussian source marginals are also allowed, and every node has constant own-factor curvature $\kappa_i(x_i)=-1/\tau_i^2$.

With constant own-factor curvature, a sink has $H_{j,R}=\kappa_j(X_j)$, so both scores are zero. At a nonsink, Theorem~\ref{thm:sink-identification}(i) and the law of total variance give $\Var(H_{j,R})\ge S_j(R)>0$.

To see that the inclusion is strict, let $X_1\sim\operatorname{Exp}(1)$ and
\[
X_2\mid X_1=x\sim\operatorname{Gamma}(\alpha,\text{scale}=1+x),
\qquad \alpha>4.
\]
This conditional has
\[
b_2(y)=\frac{y^{\alpha-1}}{\Gamma(\alpha)},\qquad
\eta_2(x)=-\frac{1}{1+x},\qquad
A_2(t)=-\alpha\log(-t).
\]
Here $\eta_2$ is nonlinear, while $\kappa_2(y)=-(\alpha-1)/y^2$ is nonconstant. Hence $S_2=0$ but $\Var(H_2)>0$; $\alpha>4$ ensures a finite second curvature moment. The CCS therefore identifies a model for which SCORE's constant-curvature criterion fails.
\end{proof}

\subsection{Link--Index Characterization}
\label{app:link-index}

For a conditional family--link--index parameterization of $Y\mid X$, let $g_{Y,c}:=(A_Y')^{-1}$ be the canonical link, set $h:=g_{Y,c}\circ g_Y^{-1}$, and write $I_f:=f(\mathring{\mathcal X}_X)$. Then the canonical-parameter function is $\eta=h\circ f$. For any link $g$, its affine class is $[g]:=\{\alpha\circ g:\alpha\ \text{is increasing affine}\}$. Thus a canonical-class link is a member of $[g_{Y,c}]$.

\begin{corollary}[Link--index characterization]
\label{cor:link-index}
In the setting of Theorem~\ref{thm:link-index}, its characterization becomes
\[
X\to Y\ \text{identifiable}
\quad\Longleftrightarrow\quad
h\circ f\ \text{is nonlinear on }\mathring{\mathcal X}_X.
\]
\begin{enumerate}
\item[(i)] If $g_Y\in[g_{Y,c}]$, then $h$ is affine and $\eta=\lambda f+\mu$ for some $\lambda\neq0$. Hence $X\to Y$ is identifiable if and only if $f$ is nonlinear on $\mathring{\mathcal X}_X$.

\item[(ii)] If $g_Y\notin[g_{Y,c}]$, then $h$ is nonlinear on its full domain, but its restriction to $I_f$ may be affine. For a nonconstant affine index $f(x)=\alpha+\beta x$, the direction is identifiable exactly when $h$ is nonlinear on $I_f=\alpha+\beta\,\mathring{\mathcal X}_X$. For a nonlinear index, the direction fails to be identifiable exactly when $f=h^{-1}\circ a$ on $\mathring{\mathcal X}_X$ for an affine function $a$ whose values lie in $h(I_f)$.
\end{enumerate}
Every unidentifiable case also has a canonical-link representation with an affine index.
\end{corollary}

\begin{proof}
Theorem~\ref{thm:link-index} reduces the problem to determining when $\eta=h\circ f$ is affine on $\mathring{\mathcal X}_X$. If $f=h^{-1}\circ a$ for an affine function $a$, then $\eta=a$ is affine. Conversely, suppose that $h\circ f=a$ for an affine $a$. The map $h$ is strictly increasing on $I_f$, so its inverse is well defined on $h(I_f)$. Applying this inverse pointwise yields $f=h^{-1}\circ a$ on $\mathring{\mathcal X}_X$.

For clause~(i), $g_Y\in[g_{Y,c}]$ holds exactly when $h$ is affine. Therefore, $\eta=\lambda f+\mu$ with $\lambda\neq0$, and $\eta$ is nonlinear exactly when $f$ is nonlinear.

For clause~(ii), first consider a nonconstant affine index $f(x)=\alpha+\beta x$ with $\beta\neq0$. It maps $\mathring{\mathcal X}_X$ bijectively and affinely onto $I_f=\alpha+\beta\mathring{\mathcal X}_X$. Consequently, $h\circ f$ is affine on $\mathring{\mathcal X}_X$ exactly when $h$ is affine on $I_f$. For a nonlinear index, the preceding inverse argument shows that the exceptions are exactly the functions $f=h^{-1}\circ a$ with $a$ affine. Condition~\ref{cond:affine-reverse} supplies the reverse representation in the affine case. Taking $g_Y=g_{Y,c}$ then makes $f=\eta$ affine as well, which proves the final statement.
\end{proof}

\section{Proofs for Section~\ref{sec:disco-algorithm}}
\label{app:disco-proofs}

\subsection{Proof of Proposition~\ref{prop:marginal-concrete-score}}
\label{app:marginal-proof}
\begin{proof}
Let $S=V\setminus R$ and choose a regular conditional distribution of $X_S$ given $X_R$. For $P_R$-almost every $x_R$ with $p_R(x_R)>0$, Bayes' rule gives
\begin{align*}
\E[s_{j,y}^{V}(X_V)\mid X_R=x_R]
&=
\int
\frac{p_V(x_R^{j\to y},x_S)}
     {p_V(x_R,x_S)}
\frac{p_V(x_R,x_S)}{p_R(x_R)}\,d\nu_S(x_S)\\
&=\frac{p_R(x_R^{j\to y})}{p_R(x_R)}
=s_{j,y}^{R}(x_R).
\end{align*}
Here $y$ is a neighboring value in $\mathcal X_j$, and $\nu_S$ is the product of the dominating measures on the removed coordinates. This proves the stated identity.

It remains to verify integrability for the neighboring entries used by \DISCO{}. Fix $\delta\in\{-1,+1\}$. For $A\in\{R,V\}$, define the zero-extended entry
\[
\bar s_j^{A,\delta}(x_A)
=
\begin{cases}
s^A_{j,\,x_j+\delta}(x_A),&x_j+\delta\in\mathcal X_j,\\
0,&x_j+\delta\notin\mathcal X_j.
\end{cases}
\]
The support event depends only on $x_j$, so the preceding identity implies $\bar s_j^{R,\delta}(X_R)= \E[\bar s_j^{V,\delta}(X_V)\mid X_R]$. Moreover,
\[
\E[\bar s_j^{V,\delta}(X_V)]
=
\sum_{x_j:\,x_j+\delta\in\mathcal X_j}
\int p_V(x_j+\delta,x_{-j})\,d\nu_{-j}(x_{-j})
\leq1.
\]
Thus these zero-extended neighboring entries belong to $L^1$.
\end{proof}

\subsection{Proof of Theorem~\ref{thm:plugin-consistency}}
\label{app:curvature-consistency}

For count-valued $P$, let $\mathcal A(G)$ be the nonempty ancestral remaining sets of $G$, and write $H_{jj,R}:=H_{j,R}$. For $R\in\mathcal A(G)$ and $i,j\in R$, use the events from Appendix~\ref{app:notation}: $A_{jj,R}:=\mathcal A_j^{(2)}$ and $A_{ij,R}:=\mathcal A_{ij}^{(1,1)}$ for $i\ne j$. Set $Q_{ij,R}:=P_R(\,\cdot\mid A_{ij,R})$; the support assumptions give $P_R(A_{ij,R})>0$.

\paragraph{Training and evaluation samples.}
Let $\mathcal F_n$ contain the training sample of size $n$, the fitting randomness, and a finite nonempty noise grid $\mathcal G_{\sigma,n}\subset(0,\infty)$. For every $\sigma\in\mathcal G_{\sigma,n}$, the noise-specific estimated curvatures $\widehat H_{ij,R}^{(n)}(\cdot;\sigma)$ are $\mathcal F_n$-measurable. They use ordinary, unclipped shifts and are evaluated only on $A_{ij,R}$. The evaluation observations $X^{(1)},\ldots,X^{(m_n)}$ are iid from $P$, independent of $\mathcal F_n$, with $m_n\to\infty$ as $n\to\infty$.

\paragraph{Estimators and graph threshold.}
For parent selection, let
\[
\mathcal T(G):=\{(R,i,j):R\in\mathcal A(G),\ i\text{ is a sink of }G_R,
\ j\in R\setminus\{i\}\},
\]
and write $O_{j\to i}(R):=\E_{Q_{ij,R}}|H_{ij,R}|$. For each triple, some causal order places $R\setminus\{i\}$ first and $i$ next, so Theorem~\ref{thm:sink-identification}(ii) applies. We use the $L^2$ and $L^1$ moment conditions of Appendix~\ref{app:score-regularity} for all diagonal entries with $R\in\mathcal A(G)$, $j\in R$, and all triples in $\mathcal T(G)$, respectively.

Write $\overline H_{ij,R}^{(n)}:=|\mathcal G_{\sigma,n}|^{-1}\sum_{\sigma\in\mathcal G_{\sigma,n}}\widehat H_{ij,R}^{(n)}(\cdot;\sigma)$. Let $\widehat S_{j,n}(R)$ be the grouped estimator \eqref{eq:cond-var-score} applied to $\overline H_{jj,R}^{(n)}$ on $A_{jj,R}$, with fixed $n_{\min}=k\ge2$ and unbiased within-group sample variances. Let $\widehat O_{j\to i,n}(R)$ use the averaging rule in \eqref{eq:empirical-ocs} for marginal $R$ and noise grid $\mathcal G_{\sigma,n}$, restricting observations to $A_{ij,R}$ and using the evaluation sample as one fold; absolute values are taken before noise averaging. Empty evaluation averages and maxima are zero. Given $\widehat\pi_n$, write $\widehat R_{i,n}:=\{i\}\cup\mathrm{Pred}_{\widehat\pi_n}(i)$. The graph $\widehat G$ includes $j\to i$ exactly when $j\in\mathrm{Pred}_{\widehat\pi_n}(i)$ and $\widehat O_{j\to i,n}(\widehat R_{i,n})>\lambda_n$. The procedure may be completed arbitrarily outside $\mathcal A(G)$.

\begin{assumption}[Marginal curvature estimation]
\label{ass:unified-curvature-consistency}
Under the setup above, the noise-specific fitted marginal curvatures converge uniformly over the evaluation noise grid to the clean curvatures:
\begin{equation}
\label{eq:unified-curvature-consistency}
\varepsilon_n
:=\max_{\substack{R\in\mathcal A(G),\ i,j\in R\\
                  \sigma\in\mathcal G_{\sigma,n}}}
\bigl\|\widehat H_{ij,R}^{(n)}(\cdot;\sigma)-H_{ij,R}\bigr\|_{L^2(Q_{ij,R})}
\xrightarrow{p}0,
\end{equation}
with the error norms finite almost surely.
\end{assumption}

\begin{proof}
We first show
\begin{align}
\max_{R\in\mathcal A(G)}\max_{j\in R}
|\widehat S_{j,n}(R)-S_j(R)|&\xrightarrow{p}0,
\label{eq:empirical-ccs-consistency}\\
\max_{(R,i,j)\in\mathcal T(G)}
|\widehat O_{j\to i,n}(R)-O_{j\to i}(R)|&\xrightarrow{p}0.
\label{eq:empirical-ocs-consistency}
\end{align}
\paragraph{Step 1: Consistency of the CCS estimator.}
For the CCS, fix $R,j$ and write $Q=Q_{jj,R}$, $Z=X_j$, $f=H_{jj,R}\in L^2(Q)$, and $\mathsf S_N(g)$ for the grouped estimator \eqref{eq:cond-var-score} using curvature $g$ on the $N$ admissible observations. Here $N\to\infty$ in probability, and conditional on $N$ these observations are iid from $Q$. Let $N_v$ be the count at level $v$ and $D_N=\sum_v N_v\mathbf1\{N_v\ge k\}$. For any finite set $F$ of positive-probability levels, each level is eventually retained, giving $\liminf_N D_N/N\ge Q(Z\in F)$ almost surely. Taking $F$ to the countable support yields $D_N/N\to1$; this also holds in probability at the random admissible count.

The contribution of levels in $F$ to $\mathsf S_N(f)$ converges to $\sum_{v\in F}Q(Z=v)\Var_Q(f\mid Z=v)$. On $D_N>0$, the remaining contribution is bounded by
\[
\frac{c_k}{D_N}\sum_{a=1}^N
f(X_R^{(a)})^2\mathbf1\{Z_a\notin F\},
\qquad c_k:=\frac{k}{k-1}\le2.
\]
By the law of large numbers and $D_N/N\to1$, this bound converges to $c_k\E_Q[f^2\mathbf1\{Z\notin F\}]$, which tends to zero as $F$ increases. The population tail has the same bound, so $\mathsf S_N(f)\to_p\E_Q[\Var_Q(f\mid Z)]=S_j(R)$.

Now put $f_n=\overline H_{jj,R}^{(n)}$ and $h_n=f_n-f$. The triangle inequality and Assumption~\ref{ass:unified-curvature-consistency} give
\[
\begin{aligned}
\|h_n\|_{L^2(Q)}
&\le\frac{1}{|\mathcal G_{\sigma,n}|}
\sum_{\sigma\in\mathcal G_{\sigma,n}}
\bigl\|\widehat H_{jj,R}^{(n)}(\cdot;\sigma)-H_{jj,R}\bigr\|_{L^2(Q)}\\
&\le\varepsilon_n.
\end{aligned}
\]
Since $\sqrt{\mathsf S_N(\cdot)}$ is a seminorm, on $D_N>0$,
\[
\bigl|\sqrt{\mathsf S_N(f_n)}-\sqrt{\mathsf S_N(f)}\bigr|^2
\le\mathsf S_N(h_n)
\le c_k\frac{N}{D_N}T_n,
\qquad T_n:=\frac1N\sum_{a=1}^N h_n(X_R^{(a)})^2,
\]
with $T_n=0$ if $N=0$. Independence gives $\E[T_n\mid\mathcal F_n]\le\varepsilon_n^2$. For every $t>0$, conditional Markov's inequality gives
\[
\Pr(T_n>t)\le\E[\min\{1,\varepsilon_n^2/t\}]\longrightarrow0,
\]
because the bounded integrand converges to zero in probability. Together with $D_N/N\to1$, this proves $\widehat S_{j,n}(R)\to_p S_j(R)$ and, by finiteness of the index set, \eqref{eq:empirical-ccs-consistency}.

\paragraph{Step 2: Consistency of the OCS estimator.}
For the OCS, fix $(R,i,j)\in\mathcal T(G)$ and let $\widetilde O_{j\to i,n}(R)$ average $|H_{ij,R}|$ on the same admissible observations. The $L^1(Q_{ij,R})$ moment and the law of large numbers give $\widetilde O_{j\to i,n}(R)\to_p O_{j\to i}(R)$, while
\[
\begin{aligned}
&\E\!\left[
|\widehat O_{j\to i,n}(R)-\widetilde O_{j\to i,n}(R)|
\mid\mathcal F_n\right]\\
&\quad\le\frac{1}{|\mathcal G_{\sigma,n}|}
\sum_{\sigma\in\mathcal G_{\sigma,n}}
\bigl\|\widehat H_{ij,R}^{(n)}(\cdot;\sigma)-H_{ij,R}\bigr\|_{L^1(Q_{ij,R})}\\
&\quad\le\varepsilon_n.
\end{aligned}
\]
The first inequality uses $\bigl||a|-|b|\bigr|\le|a-b|$ before averaging over noise levels; the second uses the $L^2$ error bound. The same conditional Markov argument and finiteness of $\mathcal T(G)$ yield \eqref{eq:empirical-ocs-consistency}.

\paragraph{Step 3: Recovery of a causal order.}
By Theorem~\ref{thm:sink-identification}(i), sinks have zero CCS and nonsinks have positive CCS on every ancestral set. If nonsinks exist, their scores have a positive minimum $\Delta$ over this finite collection. When all CCS errors are below $\Delta/3$, every empirical minimizer is a sink. Starting from $V$, each removal therefore preserves ancestrality and yields a causal order. By \eqref{eq:empirical-ccs-consistency}, this event has probability tending to one. If no nonsinks exist, every order is valid.

\paragraph{Step 4: Recovery of the DAG.}
For a nonparent triple in $\mathcal T(G)$, Theorem~\ref{thm:sink-identification}(ii) gives $H_{ij,R}=0$ $Q_{ij,R}$-almost surely, so $\widetilde O_{j\to i,n}(R)=0$. Hence
\[
\Pr\{\widehat O_{j\to i,n}(R)>\lambda_n\}
\le\E[\min\{1,\varepsilon_n/\lambda_n\}]\longrightarrow0.
\]
For a true parent, $O_{j\to i}(R)>0$, so \eqref{eq:empirical-ocs-consistency} and $\lambda_n\to0$ give selection with probability tending to one. A union bound over $\mathcal T(G)$, together with the correct-order event, proves $\Pr\{\widehat G=G\}\to1$.
\end{proof}

\begin{remark*}[Scope of the consistency result]
Consistency of common-grid clipping and median/IQR standardization is not established here. Applying the theorem to fitted grids requires control of rank and truncation errors, and Assumption~\ref{ass:unified-curvature-consistency} requires any evaluation-noise bias to vanish.
\end{remark*}

\subsection{Proof of Proposition~\ref{prop:support-index-invariance}}
\begin{proof}
Let $Y_j=\phi_j(X_j)$ with $\phi_j$ strictly increasing, and write the ordered support as $v_{j,0}<v_{j,1}<\cdots$. Strict monotonicity gives $T_j^Y(\phi_j(v_{j,k}))=k=T_j^X(v_{j,k})$ for every support value. Thus $T^Y(Y)=T^X(X)$ almost surely, so their marginal CCS and OCS values, and hence their causal orders and parent sets, coincide. The fitted transform also depends only on this order, including its clipping and between-value rules. It therefore produces identical transformed training and evaluation inputs, proving the fixed-seed claim.
\end{proof}

\section{Theoretical Properties of Concrete-Score Estimation}
\label{app:score-estimation-proofs}

Fix a remaining set $R$ and a noise level $\sigma>0$. Starting from $x_0\sim p_R$, corrupt the coordinates independently with kernels $q_\sigma^j(b\mid a)=[\exp(\sigma Q^j)]_{ab}$, where $Q^j$ is a continuous-time Markov-chain rate matrix. We assume $q_\sigma^j(b\mid a)>0$ for every pair of states $a,b$ in the coordinate state space and every positive noise level used below. The finite-grid reflecting birth--death kernel used in \DISCO{} satisfies this condition. The resulting transition probability and noisy marginal are
\[
p_{\sigma\mid0}(x_R\mid x_0)
=\prod_{j\in R}q_\sigma^j(x_j\mid x_{0j}),
\qquad
p_{R,\sigma}(x_R)
=\sum_{x_0}p_{\sigma\mid0}(x_R\mid x_0)p_R(x_0).
\]
The corresponding noisy concrete score is $s_{j,y}^{R}(x_R;\sigma)=p_{R,\sigma}(x_R^{j\to y})/p_{R,\sigma}(x_R)$.

\subsection{Denoising Score-Entropy Objective}
\label{app:denoising-objective}

For $u,v>0$, the score-entropy loss is
\begin{equation}
\label{eq:score-entropy-loss}
\ell_{\mathrm{SE}}(u,v)=u-v+v\log(v/u).
\end{equation}

\paragraph{Denoising objective.} Using the corruption and noisy concrete-score notation above, the count-structured loss below is the birth--death restriction of the denoising score-entropy objective of \citet{lou2024discrete}. Each $Q^j$ has nonnegative off-diagonal entries and zero row sums. We estimate each non-self entry $y\neq x_j$ by $\bigl[s_\theta(x_R;\sigma)\bigr]_{j,y}$, fix the self entry to $1$, and minimize
\begin{equation}
\label{eq:dse}
\begin{aligned}
\mathcal{L}_{\mathrm{DSE}}(\theta)
={}&
\E_{t}\,\E_{x_0\sim p_R}\,
\E_{x_R\sim p_{\sigma\mid 0}(\cdot\mid x_0)}
\sum_{j\in R}\sum_{y\neq x_j} w_{x_j y}\\
&\times\dot\sigma(t)\,\ell_{\mathrm{SE}}\!\left(
\bigl[s_\theta(x_R;\sigma)\bigr]_{j,y},\rho^{x_0}_{j,y}
\right),
\end{aligned}
\end{equation}
where $\ell_{\mathrm{SE}}$ is defined in \eqref{eq:score-entropy-loss}, $t\sim\mathrm{Unif}(\epsilon_t,1)$ with $\epsilon_t\in(0,1)$, and $\sigma=\sigma(t)>0$ with $0<\dot\sigma(t)<\infty$ on the sampled interval. We set $\rho^{x_0}_{j,y}=q_\sigma^j(y\mid x_{0j})/q_\sigma^j(x_j\mid x_{0j})$ and $w_{x_j y}=Q^j(x_j,y)$, matching the training specification in Appendix~\ref{app:training}.

\paragraph{Population minimizer.}
Restrict the objective to modeled entries with $w_{x_jy}>0$ and omit entries with zero generator rate. At a fixed sampled time $t$, for a corrupted state $x_R$ and modeled entry $(j,y)$ with $p_{R,\sigma}(x_R)>0$ and $p_{R,\sigma}(x_R^{j\to y})>0$, set $\bar\rho_{j,y}(x_R)=\E[\rho^{X_0}_{j,y}\mid X_R=x_R]$. Expanding the loss gives, almost surely, the finite conditional risk $u-\bar\rho_{j,y}(x_R)\log u+C$ at prediction $u>0$, where $C$ is finite and does not depend on $u$. Its derivative $1-\bar\rho_{j,y}(x_R)/u$ shows that the unique minimizer is $u=\bar\rho_{j,y}(x_R)$. The strictly positive generator and time weights do not alter this conditional minimizer. Kernel positivity justifies cancellation of the transition probabilities, and the product form of the corruption kernel gives
\begin{align*}
\bar\rho_{j,y}(x_R)
&=
\sum_{x_0}
\frac{p_R(x_0)p_{\sigma\mid0}(x_R\mid x_0)}{p_{R,\sigma}(x_R)}
\frac{q_\sigma^j(y\mid x_{0j})}{q_\sigma^j(x_j\mid x_{0j})}\\
&=
\frac{1}{p_{R,\sigma}(x_R)}
\sum_{x_0}p_R(x_0)p_{\sigma\mid0}(x_R^{j\to y}\mid x_0)
=
\frac{p_{R,\sigma}(x_R^{j\to y})}{p_{R,\sigma}(x_R)}
=s_{j,y}^{R}(x_R;\sigma).
\end{align*}
On the finite common corruption-model grid, both the generator weights and the number of modeled entries are bounded. These entrywise minima therefore make the corrupted concrete score the unique population minimizer of \eqref{eq:dse} on modeled entries.

\subsection{Vanishing-Noise Limit}
\label{app:vanishing-noise-limit}

\begin{proposition}[Vanishing-noise limit]
\label{prop:vanishing-noise}
Suppose $q_\sigma^j$ tends to the identity as $\sigma\downarrow0$, which holds for the finite-grid reflecting birth--death kernel used in \textup{\DISCO}. Then $p_{R,\sigma}\to p_R$ pointwise, and for every $x_R$ with $p_R(x_R)>0$ and every $y\in\mathcal X_j$,
\[
s_{j,y}^R(x_R;\sigma)
\xrightarrow[\;\sigma\downarrow0\;]{}
s_{j,y}^R(x_R).
\]
\end{proposition}

\begin{proof}
On the countable support, $p_{R,\sigma}(x_R)=\sum_{x_0}p_{\sigma\mid0}(x_R\mid x_0)p_R(x_0)$. For each fixed $x_0$, the summand converges to $\mathbf 1\{x_R=x_0\}p_R(x_0)$ as $\sigma\downarrow0$. It is bounded by the summable function $p_R(x_0)$ because $0\leq p_{\sigma\mid0}(x_R\mid x_0)\leq1$. Dominated convergence therefore gives $p_{R,\sigma}(x_R)\to p_R(x_R)$ for every $x_R$. If $p_R(x_R)>0$, continuity of the ratio yields
\[
s_{j,y}^R(x_R;\sigma)
=\frac{p_{R,\sigma}(x_R^{j\to y})}{p_{R,\sigma}(x_R)}
\longrightarrow
\frac{p_R(x_R^{j\to y})}{p_R(x_R)}
=s_{j,y}^R(x_R).
\]
\end{proof}

\section{\DISCO{} Implementation}
\label{app:algorithm-details}

\paragraph{Cross-fitting.}
We partition the observations into evaluation folds $\mathcal I_1,\ldots,\mathcal I_B$. For each fold, we fit the rank transform, joint concrete-score network, and projection network on the other folds. We then apply the fitted transform and networks to the held-out fold to compute the CCS and OCS.

\subsection{Rank Transform}
\label{app:rank-transform}

For coordinate $j$, let $v_{j,0}<\cdots<v_{j,K_j}$ be the distinct values observed in the score-training split and set $\widehat T_j(v_{j,k})=k$. Apply this fitted transform to the held-out evaluation split without using its empirical distribution. A held-out value below $v_{j,0}$ is assigned rank $0$, a value above $v_{j,K_j}$ is assigned rank $K_j$, and a value strictly between $v_{j,k}$ and $v_{j,k+1}$ is assigned rank $\max(k,1)$. The first interior gap is assigned rank $1$ to keep the observed minimum as a separate state. Thus the evaluation values are clipped or coarsened onto the rank grid according to order rather than Euclidean distance.

\subsection{Joint Concrete-Score Training}
\label{app:training}

\paragraph{Training objective.} We train the joint concrete-score network once on the score-training split, using SEDD's denoising score-entropy objective \citep{lou2024discrete} in \eqref{eq:dse} with $R=V$: each minibatch draws $t\sim\mathrm{Unif}(\epsilon_t,1)$, perturbs $x_0\sim p_V$ to $x\sim p_{\sigma(t)\mid 0}(\cdot\mid x_0)$, and weights the loss by $\dot\sigma(t)$. We use the geometric schedule $\sigma(t)=\sigma_{\min}^{1-t}\sigma_{\max}^{t}$ with $(\epsilon_t,\sigma_{\min},\sigma_{\max})=(10^{-3},10^{-3},3)$. The training objective is restricted to the neighboring replacements $y=x_j\pm1$ and masks replacements outside the common corruption-model grid $\{0,\ldots,K\}$, where $K=\max_jK_j$. We choose a reflecting birth--death corruption on this grid for every coordinate: $Q^j(a,b)=1$ for $|a-b|=1$, zero for other off-diagonal entries, and $Q^j(a,a)=-\sum_{b\ne a}Q^j(a,b)$. The common upper bound $K$ may exceed a coordinate's largest training rank $K_j$, so this grid is not the coordinate's population support.

\paragraph{Network architecture.} The count network is a multilayer perceptron with four hidden layers of width 512 and SiLU activations, without hidden-layer normalization or dropout. Writing $K^\vee=\max(K,1)$ for the common grid bound, its value features are $[\,x/K^\vee,\ \log(1+x)/\log(1+K^\vee)\,]$. These features are standardized using training-split means and standard deviations, with standard deviations floored at $10^{-4}$; $\log\sigma$ is then appended. The output is reshaped to $\mathbb{R}^{d\times2}$, with the last axis indexing the two directions.

\subsection{Marginal Concrete-Score Projection}
\label{app:masking}
\label{app:local-score-projection}

\subsubsection{Mask Distribution}
The score-entropy estimator draws a remaining-set size uniformly from $\{2,\ldots,d\}$ and then a subset uniformly conditional on that size, using one mask per training observation and epoch. Independently with probability $0.2$, it replaces the sampled mask by the full set $V$. For $d\ge2$, the resulting distribution is
\[
\mathcal Q(R)=
\frac{0.8}{(d-1)\binom{d}{|R|}}+0.2\,\mathbf 1\{R=V\},
\qquad 2\le |R|\le d.
\]
Thus every subset of size at least two has positive sampling probability; for $d>2$, neither subsets nor final mask sizes are uniformly distributed.

\subsubsection{Score-Entropy Training}
Each projection update draws $X_0\sim p_V$, samples $\log\sigma$ uniformly from $[\log10^{-3},\log3]$, and corrupts the full observation to $X_\sigma\sim p_{\sigma\mid0}(\cdot\mid X_0)$. The joint network evaluates $X_\sigma$ without parameter updates, while the projection network receives $(X_{\sigma,R},R,\sigma)$ with removed coordinates zero-masked. Masks are independent of the observation, and projection training uses no $\dot\sigma$ loss multiplier.

A mask-conditioned MLP with the same hidden architecture as the joint network learns a log correction to a positive baseline. The baseline is obtained by filling removed coordinates with values from one training observation sampled once and held fixed, then evaluating the joint network.

We apply the score-entropy loss in \eqref{eq:score-entropy-loss} to the numerically stabilized, baseline-normalized projected and joint-network scores. Log clipping and positive numerical floors are used for stability; their constants are specified in the code. We average the loss over neighboring replacements within the grid for $j\in R$ and add a full-set ($R=V$) loss with weight~1.

\subsection{Curvature Estimation}
\label{app:two-pass}

\paragraph{Evaluation.} We use five geometrically spaced noise levels from $10^{-3}$ to $5\times10^{-3}$ in $\mathcal G_\sigma$ and direct log concrete-score outputs from the projection network.

We evaluate diagonal curvature only when $x_j+2\le K$, so both upward log concrete-score entries correspond to transitions inside the fitted model grid. Invalid entries are excluded separately for each node, rather than set to zero or removed from every node's evaluation sample. Counts and weights in \eqref{eq:cond-var-score} use admissible observations; a node with no repeated admissible level is reported as unavailable and stops ordering. This model-grid restriction does not identify the fitted grid with the population support.

For the grouped conditional variance \eqref{eq:cond-var-score}, we use rank levels with at least $n_{\min}=2$ admissible evaluation observations and weight each by the number of these observations. We average the resulting scores across folds.

\paragraph{Symmetric off-diagonal curvature evaluation.} For parent selection, we use the commuting-difference identity
\begin{equation}
\label{eq:symmetric-parent-curvature}
H_{ij,\widehat R_i}
=\Dop_i\Dop_j\log p_{\widehat R_i}
=\Dop_j\Dop_i\log p_{\widehat R_i}
=H_{ji,\widehat R_i}.
\end{equation}
For each node $i$, we obtain the $\widehat R_i$-marginal concrete score for $\widehat R_i=\{i\}\cup\mathrm{Pred}_{\widehat\pi}(i)$, evaluate the projection network on the unshifted batch and the batch shifted in coordinate $i$, $x_{\widehat R_i}+e_i$ on admissible observations, and extract the entries for all predecessor coordinates $j\in\mathrm{Pred}_{\widehat\pi}(i)$ from those two outputs. Thus the off-diagonal curvatures in \eqref{eq:estimated-curvatures} require only one shifted batch per node. We compute these curvatures for all predecessors simultaneously, averaging their absolute values over noise levels and admissible observations for each edge. All candidate parents use the same unshifted score predictions and remaining-set mask.

\subsection{Parent Selection}
\label{app:parent-rule}

We estimate OCS with $B=3$ cross-fitting folds. We use $\mathcal I_b^{ij}=\{a\in\mathcal I_b:X_{ai}+1\le K,\ X_{aj}+1\le K\}$ in \eqref{eq:empirical-ocs}, with denominator $|\mathcal I_b^{ij}|$ for each edge. An empty set is reported as unavailable and stops parent selection. Clipped-boundary evaluation uses $\mathcal I_b^{ij}=\mathcal I_b$. For a node $i$ with at least two candidate parents, we robustly standardize these estimates within each fold,
\[
Z_{j\to i}^{(b)}
=
\frac{\widehat{\mathrm{OCS}}_{j\to i}^{(b)}
-\operatorname{median}_{k\in\mathrm{Pred}_{\widehat\pi}(i)}\widehat{\mathrm{OCS}}_{k\to i}^{(b)}}
{s_i^{(b)}}.
\]
Here $s_i^{(b)}$ is the IQR of the candidate OCS estimates, falling back successively to the scaled median absolute deviation, standard deviation, and unit scale when degenerate. We apply the selection-frequency rule in Section~\ref{sec:dag-recovery}.

We use $(\tau_Z,\pi_{\min})=(2,1)$ unless otherwise stated. For a single candidate, we omit centering and use the median robust scale across children with at least two candidates in the same fold. If none exist, we use the robust scale of all candidate scores in that fold. The same threshold and selection-frequency rule still apply. Candidate sets of size two or three remain unable to pass the default threshold in exact arithmetic. This practical rule does not provide a finite-sample error-control guarantee. Appendix~\ref{app:diagnostic-parent-selection} evaluates threshold sensitivity and compares parent-selection methods.

\section{Experimental Details and Additional Results}
\label{app:disco-exp}

\subsection{Experimental Setup}
\label{app:experimental-setup}

\paragraph{Graph and sampling protocol.}
For count-data synthetic experiments, ER$k$ denotes an Erd\H{o}s--R\'enyi DAG obtained from a random permutation used as the causal order, with each admissible earlier-to-later edge included independently with probability $2k/d$ (expected edge count $k(d-1)$). Unless stated otherwise, we use ER3 DAGs. Within each setting and seed, all methods receive the same graph and observations. For sample-size curves, smaller samples are row prefixes of the largest sample. The continuous-data comparison uses the fixed-edge-count convention specified in Appendix~\ref{app:continuous-recovery}.

\paragraph{Replication and pairing.}
Unless stated otherwise, we use 10 paired seeds for each synthetic experiment. The Lahman comparisons use three runs on each fixed cohort; these are computational repetitions, not independent sampled cohorts. Ablations also share learned scores whenever applicable; graph seeds are the independent replication units.

\paragraph{Evaluation metrics.}
We use $A_{\mathrm{top}}=|\mathcal E|^{-1}\sum_{(u,v)\in\mathcal E}\mathbf 1\{\widehat\pi(u)<\widehat\pi(v)\}$ for ordering agreement. The nonempty evaluation set $\mathcal E$ is the true DAG edge set $E$ for synthetic data and the partial directed reference set $\mathcal R$ for real data. DAG recovery comparisons report $F_1$, precision, recall, and structural Hamming distance (SHD), defined as $\mathrm{FP}+\mathrm{FN}$, so reversing a true edge counts twice.

\paragraph{Training settings.}
On count data, we train the joint network for up to 800 epochs and the projection network for 800 epochs, unless stated otherwise. The COM--Poisson experiment (Table~\ref{tab:nonstandard-count-recovery}), parent-selection experiments (Tables~\ref{tab:parent-selection-sensitivity} and~\ref{tab:parent-method-comparison}), and 2019--2025 Lahman \DISCO{} analysis use 400 projection epochs. In the count-data experiments, \DISCO{} clips upper-grid shifts to $K$ and includes boundary outputs, using all fold observations for OCS, except where admissible-domain evaluation is specified. Separate continuous-data settings are given in Appendix~\ref{app:continuous-recovery}.

Both networks use Adam with batch size 512. The joint and projection learning rates start at $10^{-3}$ and $1.5\times10^{-3}$, respectively, and follow cosine schedules to $10^{-5}$ and $3\times10^{-5}$; their gradient-norm clipping thresholds are 1 and 5. Joint validation uses a 10\% internal holdout (at most 4096 observations), begins at epoch 300, and runs every 100 epochs with patience 2 and minimum improvement $10^{-4}$.

\paragraph{Data generation.}
For non-source nodes, write the index as $z_j=b_0+\sum_{k\in\pa_G(j)}\beta_{kj}\phi(X_k)$. Table~\ref{tab:dgp-spec} gives the six benchmark mechanisms; $\mu_j$ denotes the conditional mean and $q_j$ the binomial success probability. The intercept is fixed within each mechanism. We do not standardize the generated counts, normalize by in-degree, or modify observations after generation.

\begin{table}[htbp]
\TableStyle
\caption{The six single-family benchmark mechanisms. Coefficients are drawn independently and uniformly from the stated intervals.}
\label{tab:dgp-spec}
\begin{tabular}{@{}llccc@{}}
\toprule
\shortstack{Conditional\\family} & Conditional response & $\phi(x)$ & $b_0$ & Coefficients \\
\midrule
\multirow{2}{*}{Poi} & $\mu_j=0.2+\operatorname{softplus}(z_j)$ & $x$ & $0.5$ & $[0.30,0.45]$ \\
 & $\mu_j=\exp(\widetilde z_j)$ & $\log(1+x)$ & $0.4$ & $[0.15,0.30]$ \\
\midrule
\multirow{2}{*}{NB} & $\mu_j=0.2+\operatorname{softplus}(z_j)$ & $x$ & $0.5$ & $[0.30,0.45]$ \\
 & $\mu_j=\exp(\widetilde z_j)$ & $\log(1+x)$ & $0.4$ & $[0.15,0.30]$ \\
\midrule
\multirow{2}{*}{Bin} & $q_j=\Phi(z_j)$ & $x/50$ & $-1.28$ & $[0.50,1.00]$ \\
 & $q_j=\operatorname{sigmoid}(z_j)$ & $\sqrt{x/50}$ & $-2.197$ & $[0.60,0.90]$ \\
\bottomrule
\end{tabular}
\end{table}

Here $\widetilde z_j=\operatorname{clip}(z_j,-3,3.5)$ limits the exponential response, and $\Phi$ is the standard normal CDF. Source means are 2 for the softplus mechanisms and 1.5 for the exponential mechanisms. Negative binomial nodes have size $r=6$. Binomial nodes have size $M=50$ and source probability $0.10$; generated probabilities are clipped to $[10^{-8},1-10^{-8}]$. Mixed-family benchmarks use the softplus/affine rows for Poisson and negative binomial children and the sigmoid/square-root row for binomial children, with the feature map determined by the child's conditional family.

\paragraph{CCS comparison.}
\label{app:ccs-diagnostics}
Table~\ref{tab:ccs-ablation}(a) uses the softplus/affine Poisson and negative-binomial mechanisms and the sigmoid/square-root binomial mechanism in Table~\ref{tab:dgp-spec}, at $d=50$ and $n=5000$. CCS and the constant-curvature criterion use the same fitted networks. Across 10 seeds, we evaluate nested prefixes of the true causal order with sizes 2--9, excluding prefixes in which every node is a sink; this leaves 60 matched nontrivial ancestral remaining sets per conditional family. Accuracy is the fraction of sets for which the selected node is a true sink; sets within a graph seed are not independent replicates.

For DAG recovery, each criterion orders all nodes using the same fitted networks, and the default OCS rule then selects parents. Table~\ref{tab:ccs-full-dag} reports the results. CCS attains higher mean $A_{\mathrm{top}}$ and $F_1$ and lower mean SHD for Poisson and negative binomial data. For binomial data, the two criteria give similar $A_{\mathrm{top}}$, and the constant-curvature criterion gives slightly higher mean $F_1$ and lower mean SHD.

\begin{table}[H]
\TableStyle
\caption{DAG recovery with CCS and the constant-curvature criterion ($d=50$, $n=5000$). Entries are means $\pm$ sample standard deviations over 10 seeds. Constant denotes the constant-curvature criterion.}
\label{tab:ccs-full-dag}
\begin{tabular}{@{}llccc@{}}
\toprule
\shortstack{Conditional\\family} & Criterion & $A_{\mathrm{top}}$ & $F_1$ & SHD \\
\midrule
\multirow{2}{*}{Poi} & Constant & $0.915\pm0.022$ & $0.883\pm0.037$ & $31.5\pm9.4$ \\
 & CCS & $0.931\pm0.017$ & $0.892\pm0.034$ & $28.9\pm8.7$ \\
\midrule
\multirow{2}{*}{NB} & Constant & $0.907\pm0.019$ & $0.763\pm0.070$ & $57.9\pm15.5$ \\
 & CCS & $0.931\pm0.017$ & $0.792\pm0.066$ & $51.2\pm15.4$ \\
\midrule
\multirow{2}{*}{Bin} & Constant & $0.899\pm0.028$ & $0.884\pm0.027$ & $30.8\pm7.7$ \\
 & CCS & $0.901\pm0.027$ & $0.877\pm0.031$ & $32.3\pm8.4$ \\
\bottomrule
\end{tabular}
\end{table}

\subsubsection{Baseline Methods}
\label{app:baseline-methods}

\paragraph{ODS.}
\ODSpoi{} follows the QVF overdispersion ordering criterion of \citet{park2018qvf}.  The multivariate experiments use regression estimates of the conditional mean and variance because exact-match conditioning cells collapse at the graph densities considered here.  The fixed-family ODS variants (Poisson, negative binomial, and binomial) and ODS (Oracle) differ only in the supplied QVF coefficient and family-specific regression link. For ODS (Oracle), the fixed parameters are the negative binomial size $r$ and binomial trial count $M$, giving QVF coefficients $1/r$ and $-1/M$, respectively; Poisson has coefficient zero and requires no additional fixed parameter.

\paragraph{\NOTEARSls.}
We use the gCastle nonlinear MLP \citep{zhang2021gcastle} with squared loss on standardized $\log(1+x)$ counts.

\paragraph{DiffAN.}
We use the DiffAN implementation of \citet{sanchez2023diffan}, which includes CAM pruning \citep{buhlmann2014cam}. Runtime includes diffusion ordering and parent pruning.

\paragraph{MRS.}
We use the MRS implementation \citep{park2019ghd} with GES skeleton estimation. Mixed-family and scalability experiments use its Poisson setting. Single-family experiments use the Poisson and binomial settings for the corresponding families, and a hyper-Poisson proxy for negative binomial data.

\subsection{Additional Benchmarks}
\label{app:additional-benchmarks}

\subsubsection{Invariance to Strictly Increasing Transformations}
\label{app:rank-transform-invariance}

Table~\ref{tab:rank-transform-invariance} compares \DISCO{} with \DISCONoRank{} after applying $\log(1+x)$, Anscombe's $2\sqrt{x+3/8}$ transform, and $x^2+x$ to the same Poisson data, generated by the softplus/affine mechanism in Table~\ref{tab:dgp-spec}. With the rank transform, median $A_{\mathrm{top}}=0.922$ and $F_1=0.894$ remain unchanged across transformations. Without it, median $F_1$ falls from $0.899$ on raw counts to $0.301$ under $x^2+x$.

\begingroup
\begin{table}[H]
\TableStyle
\caption{Invariance to strictly increasing transformations on Poisson DAGs ($d=50$, $n=5000$). Entries are medians. \DISCONoRank{} omits the rank transform.}
\label{tab:rank-transform-invariance}
\begin{minipage}{\linewidth}
\normalsize\rmfamily
\setlength{\parskip}{0pt}
\setlength{\tabcolsep}{0pt}
\renewcommand{\arraystretch}{1.0}
\setlength{\aboverulesep}{0.3ex}
\setlength{\belowrulesep}{0.4ex}
\begin{tabular*}{\linewidth}{@{\extracolsep{\fill}}lcc@{\hspace{5pt}}cc@{\hspace{5pt}}cc@{\hspace{5pt}}cc@{}}
\toprule
& \multicolumn{2}{c}{Raw counts}
& \multicolumn{2}{c}{$\log(1+x)$}
& \multicolumn{2}{c}{Anscombe}
& \multicolumn{2}{c}{$x^2+x$} \\
\cmidrule(lr){2-3}\cmidrule(lr){4-5}\cmidrule(lr){6-7}\cmidrule(lr){8-9}
\makebox[0.18\linewidth][l]{Method}
& \makebox[0.095\linewidth]{$A_{\mathrm{top}}$} & \makebox[0.095\linewidth]{$F_1$}
& \makebox[0.095\linewidth]{$A_{\mathrm{top}}$} & \makebox[0.095\linewidth]{$F_1$}
& \makebox[0.095\linewidth]{$A_{\mathrm{top}}$} & \makebox[0.095\linewidth]{$F_1$}
& \makebox[0.095\linewidth]{$A_{\mathrm{top}}$} & \makebox[0.095\linewidth]{$F_1$} \\
\midrule
\textbf{\DISCO{}} & 0.922 & 0.894 & 0.922 & 0.894 & 0.922 & 0.894 & 0.922 & 0.894 \\
\DISCONoRank{} & 0.937 & 0.899 & 0.897 & 0.726 & 0.890 & 0.860 & 0.834 & 0.301 \\
\bottomrule
\end{tabular*}\par
\end{minipage}
\end{table}
\endgroup

\subsubsection{Single-Family DAG Recovery}
\label{app:extended-single-family}

The binomial--probit panel uses \ODSpoi{} with a Poisson QVF, a non-oracle specification for binomial data.

\begin{figure}[htbp]
\centering
\includegraphics[width=\linewidth]{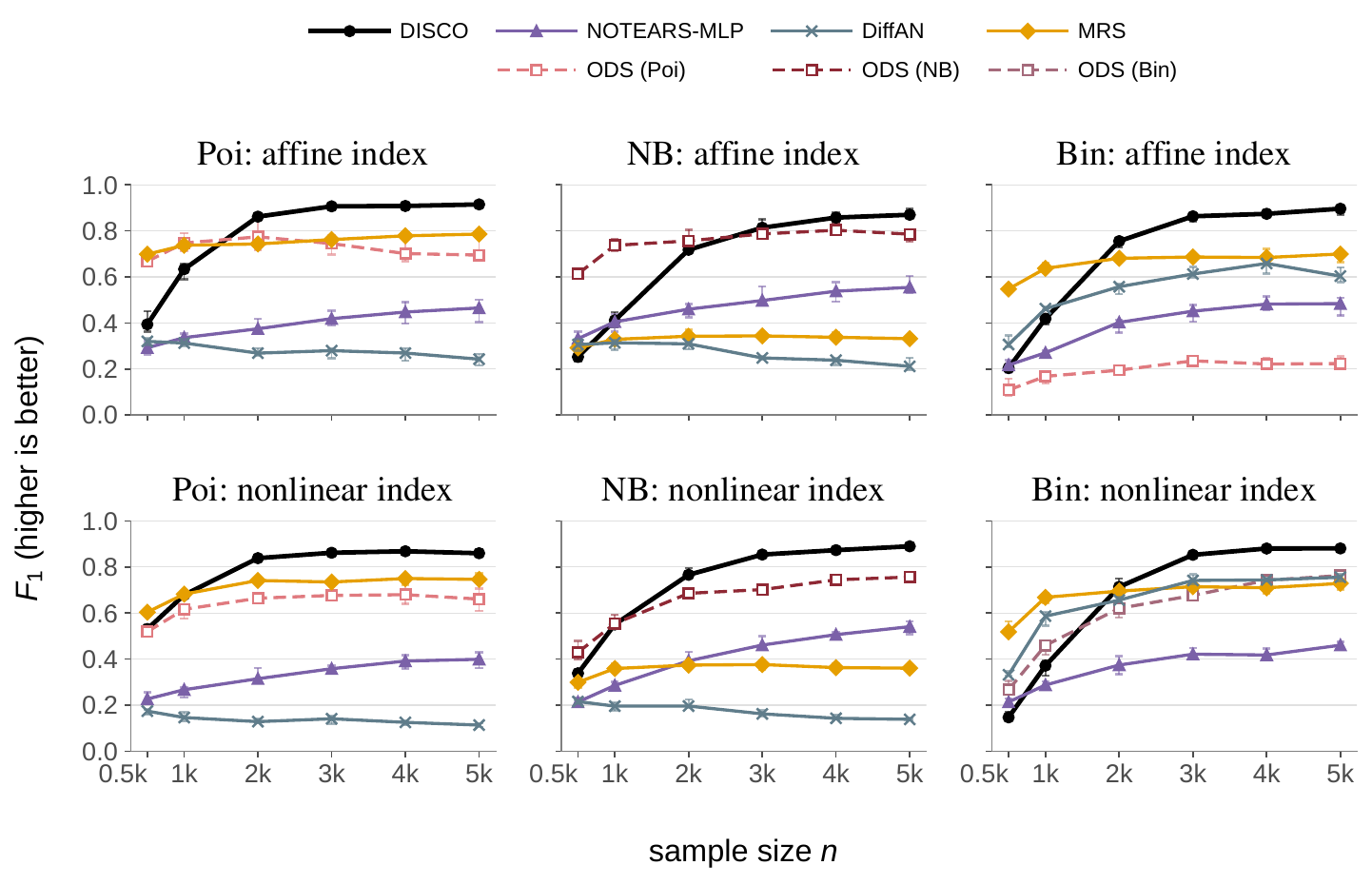}
\caption{Extended single-family comparison across all six mechanisms and six sample sizes. Columns correspond to the Poi, NB, and Bin conditional families; the top and bottom rows use the first and second mechanisms for each family in Table~\ref{tab:dgp-spec}, respectively. Points and error bars show the median and interquartile range.}
\label{fig:single-family-all-mechanisms}
\end{figure}

\subsubsection{Mixed-Family DAG Recovery}
\label{app:extended-mixed-family}

We split a causal order into three consecutive groups of nearly equal size and assign Poisson, negative binomial, and binomial mechanisms to these groups in all six permutations. The arrows in the panel titles indicate this assignment order. Each node uses the mechanism for its assigned conditional family in Table~\ref{tab:dgp-spec}. Figure~\ref{fig:curvature-scaling}(c) reports the NB $\to$ Bin $\to$ Poi configuration. To examine sensitivity to conditional-family specification, Figure~\ref{fig:mixed-family-all-methods} compares \DISCO{} with ODS (Oracle), ODS (Poi), ODS (NB), and ODS (Bin) across all six configurations.

\begin{figure}[htbp]
\centering
\includegraphics[width=\linewidth]{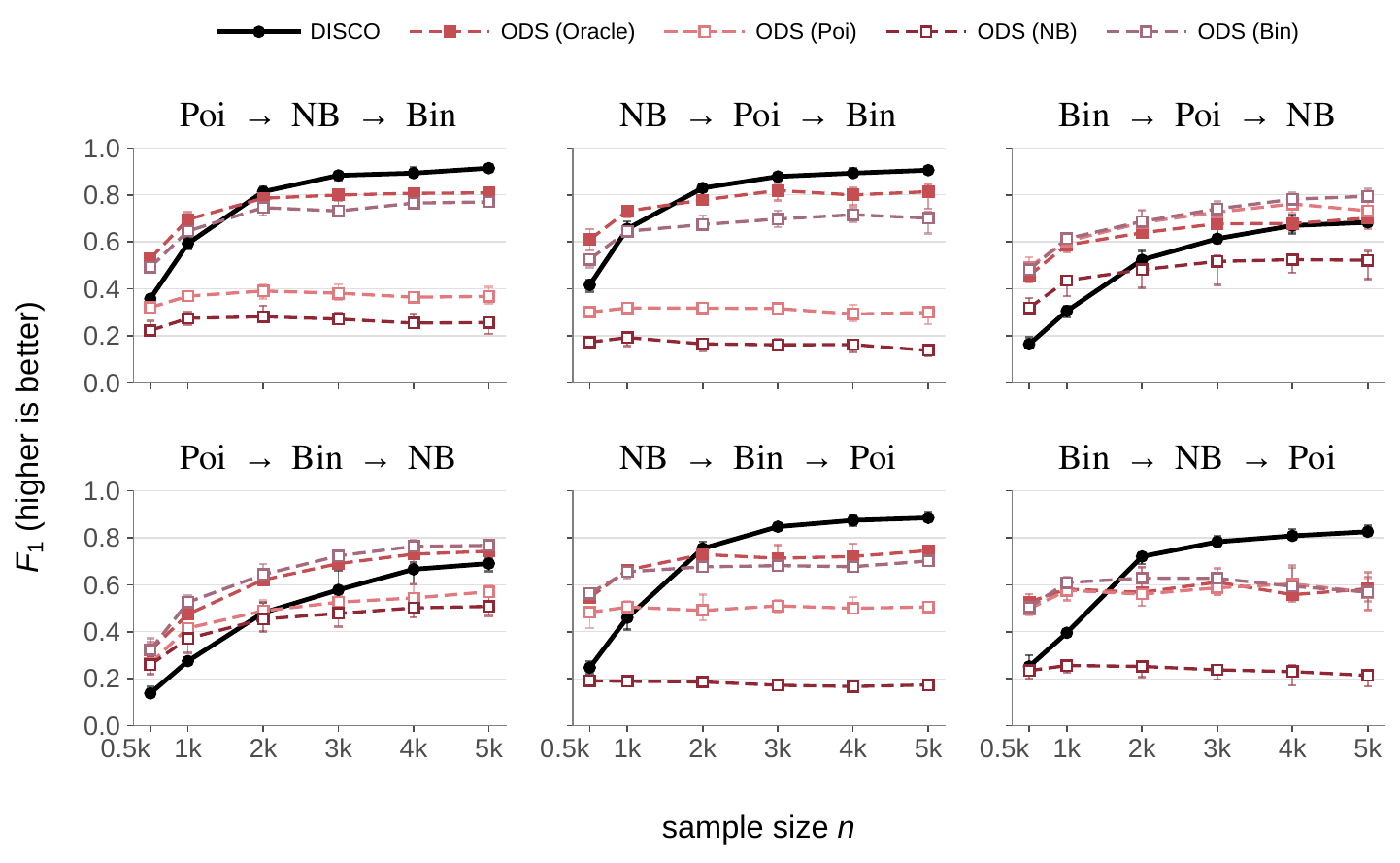}
\caption{Mixed-family DAG recovery across six conditional-family configurations. Each panel compares \DISCO{} with ODS (Oracle), ODS (Poi), ODS (NB), and ODS (Bin). ODS (Oracle) uses the true nodewise conditional families and their fixed parameters. Points and error bars show the median and interquartile range.}
\label{fig:mixed-family-all-methods}
\end{figure}

\subsubsection{Scalability with Graph Size}

Table~\ref{tab:scale} reports $A_{\mathrm{top}}$ for the Poisson scalability experiments in Table~\ref{tab:scalability-main}, using the softplus/affine mechanism in Table~\ref{tab:dgp-spec} under the same 240-minute time limit.

\begin{table}[H]
\TableStyle
\caption{Ordering accuracy in the scalability experiments at $n=5000$. Entries are medians; configurations marked Timeout in Table~\ref{tab:scalability-main} are omitted.}
\label{tab:scale}
\begin{tabular}{@{}lrc@{}}
\toprule
Method & $d$ & $A_{\mathrm{top}}$ \\
\midrule
\multirow{4}{*}{\DISCO{}} & 100 & 0.937 \\
 & 200 & 0.902 \\
 & 500 & 0.914 \\
 & 1000 & 0.908 \\
\midrule
\multirow{3}{*}{\textsc{ODS} (Oracle)} & 100 & 0.819 \\
 & 200 & 0.816 \\
 & 500 & 0.811 \\
\midrule
\multirow{3}{*}{MRS} & 100 & 0.797 \\
 & 200 & 0.779 \\
 & 500 & 0.790 \\
\midrule
\NOTEARSls{} & 100 & 0.541 \\
\bottomrule
\end{tabular}
\end{table}

Total time in Table~\ref{tab:scalability-main} covers the complete DAG recovery procedure, including parent selection. MRS, \NOTEARSls{}, and DiffAN were first evaluated with a four-hour limit. Configurations finishing within this limit were evaluated on the remaining seeds. For DiffAN at $d=100$ and $d=200$, the time limit was reached during parent pruning after ordering had finished.

CPU baselines use one thread, while \DISCO{} uses GPU training and CPU post-processing. The runtime comparison is therefore descriptive and does not imply matched hardware.

\subsubsection{Sparse DAG Recovery}
\label{app:er1-recovery}

We repeat the single-family and mixed-family designs with ER1 graphs, holding the mechanisms, coefficient distributions, and sample-size grid fixed. Table~\ref{tab:er1-recovery} reports medians for \DISCO{} at $n=5000$.

\begin{table}[H]
\TableStyle
\caption{Single-family and mixed-family DAG recovery on ER1 graphs ($d=100$, $n=5000$). Entries are medians for \DISCO{}.}
\label{tab:er1-recovery}
\label{tab:er1-single}
\label{tab:er1-mixed}
\begin{tabular}{@{}lllccc@{}}
\toprule
Setting & \shortstack{Conditional\\family} & Index/configuration & $A_{\mathrm{top}}$ & $F_1$ & SHD \\
\midrule
\multirow{6}{*}{\shortstack{Single-\\family}} & \multirow{2}{*}{Poi} & Affine & 0.816 & 0.750 & 49.0 \\
 &  & Nonlinear & 0.897 & 0.873 & 25.0 \\
 & \multirow{2}{*}{NB} & Affine & 0.841 & 0.788 & 42.5 \\
 &  & Nonlinear & 0.907 & 0.884 & 22.0 \\
 & \multirow{2}{*}{Bin} & Affine & 0.886 & 0.860 & 28.5 \\
 &  & Nonlinear & 0.904 & 0.886 & 23.0 \\
\midrule
\multirow{6}{*}{\shortstack{Mixed-\\family}} & \multicolumn{2}{l}{Poi $\to$ NB $\to$ Bin} & 0.891 & 0.845 & 29.5 \\
 & \multicolumn{2}{l}{Poi $\to$ Bin $\to$ NB} & 0.845 & 0.796 & 44.0 \\
 & \multicolumn{2}{l}{NB $\to$ Poi $\to$ Bin} & 0.810 & 0.777 & 44.0 \\
 & \multicolumn{2}{l}{NB $\to$ Bin $\to$ Poi} & 0.766 & 0.705 & 58.0 \\
 & \multicolumn{2}{l}{Bin $\to$ Poi $\to$ NB} & 0.739 & 0.685 & 67.0 \\
 & \multicolumn{2}{l}{Bin $\to$ NB $\to$ Poi} & 0.723 & 0.633 & 75.0 \\
\bottomrule
\end{tabular}
\end{table}

\subsubsection{Scale-Free DAG Recovery}
\label{app:scale-free-recovery}

We replace the ER graph generator with the Barab\'asi--Albert preferential-attachment generator, keeping $d=50$ and $n=5000$ for each of Poisson, negative binomial, and binomial. SF1 and SF3 attach each new node to one and three existing nodes, respectively, yielding 49 and 141 edges. A random node ordering orients the skeleton acyclically; attachment time is not supplied as a causal order. The conditional mechanisms and coefficient ranges are unchanged from the parent-selection study in Appendix~\ref{app:diagnostic-parent-selection}. Seed labels identify independent repetitions within a graph type, not identical observations across ER and SF graphs.

\begin{table}[H]
\TableStyle
\caption{Scale-free count-DAG recovery ($d=50$, $n=5000$). Entries are means $\pm$ sample standard deviations.}
\label{tab:scale-free-recovery}
\begin{tabular}{@{}llccc@{}}
\toprule
Graph & \shortstack{Conditional\\family} & $A_{\mathrm{top}}$ & $F_1$ & SHD \\
\midrule
\multirow{3}{*}{SF1} & Poi & $0.784\pm0.065$ & $0.705\pm0.091$ & $28.6\pm8.5$ \\
 & NB & $0.822\pm0.054$ & $0.741\pm0.094$ & $24.8\pm8.9$ \\
 & Bin & $0.778\pm0.060$ & $0.782\pm0.043$ & $20.1\pm4.0$ \\
\midrule
\multirow{3}{*}{SF3} & Poi & $0.918\pm0.019$ & $0.719\pm0.074$ & $65.3\pm14.4$ \\
 & NB & $0.904\pm0.035$ & $0.620\pm0.114$ & $82.4\pm19.7$ \\
 & Bin & $0.769\pm0.050$ & $0.693\pm0.077$ & $71.2\pm14.8$ \\
\bottomrule
\end{tabular}
\end{table}

Recovery varies with both graph structure and conditional family. In particular, the higher $A_{\mathrm{top}}$ on SF3 for Poisson and negative binomial does not imply a corresponding increase in $F_1$. These results extend the empirical evaluation beyond ER graphs; they do not establish performance independent of graph structure or conditional family.

\subsubsection{DAG Recovery Beyond Standard Count Families}
\label{app:nonstandard-count-recovery}

We test recovery beyond the Poisson, negative binomial, and binomial benchmark conditional families using mean-parametrized Conway--Maxwell--Poisson (COM--Poisson) distributions \citep{huang2017cmp}. For a parent-independent shape $\nu>0$,
\[
p(y\mid x_{\pa_G(j)})=
\exp\{y\eta(x_{\pa_G(j)})-\nu\log(y!)-A_\nu(\eta(x_{\pa_G(j)}))\},
\qquad y\in\Nzero.
\]
We use $\nu\in\{0.5,1,2\}$ on DAGs with $d=50$ and $n=5000$; $\nu=1$ reduces to the Poisson distribution. Graphs and coefficients are paired by seed. At fixed parent values, all three settings share source mean 2 and non-source conditional mean $0.2+\operatorname{softplus}(0.5+\sum_{k\in\pa_G(j)}\beta_{kj}X_k)$, with $\beta_{kj}\sim\mathrm{U}(0.30,0.45)$. We solve for the canonical parameter to match this mean, rather than treating the usual COM--Poisson rate as its expectation. Hence conditional means are matched, while the corresponding canonical parameters and samples may differ. Sampling approximates the unbounded distribution through adaptive support expansion, with bounds on the omitted probability mass and its first-moment contribution. \DISCO{} receives neither the conditional family nor $\nu$; the same configuration is used for all three $\nu$ values.

\begin{table}[H]
\TableStyle
\caption{Recovery beyond the three benchmark conditional families ($d=50$, $n=5000$). Entries are means $\pm$ sample standard deviations. The $\nu=1$ row uses results from the corresponding Poisson experiment; graph metrics include \DISCO{} parent selection.}
\label{tab:nonstandard-count-recovery}
\begin{tabular}{@{}lccc@{}}
\toprule
Conditional family & $A_{\mathrm{top}}$ & $F_1$ & SHD \\
\midrule
COM--Poisson ($\nu=0.5$) & $0.934\pm0.018$ & $0.853\pm0.033$ & $37.8\pm8.5$ \\
COM--Poisson ($\nu=1$) & $0.928\pm0.020$ & $0.877\pm0.043$ & $32.4\pm10.2$ \\
COM--Poisson ($\nu=2$) & $0.920\pm0.019$ & $0.873\pm0.026$ & $33.6\pm6.5$ \\
\bottomrule
\end{tabular}
\end{table}

Ordering remains accurate in both non-Poisson settings, while DAG recovery varies: $\nu=0.5$ has lower mean $F_1$ than $\nu=1$. This supports recovery on additional fixed-carrier one-parameter families, not arbitrary count conditional distributions or parent-dependent shape parameters.

\subsubsection{Continuous-Data Ordering Recovery}
\label{app:continuous-recovery}

\paragraph{Algorithm.}
\DISCOCont{} builds on DiffAN's score network \citep{sanchez2023diffan}, adds marginal projection, and uses CCS for sink selection. We train an unmasked joint score network for 1000 epochs and a width-256 mask-conditioned projection network for 400 epochs, using squared error against noise predictions from the fixed joint network on corrupted inputs. Both stages use a learning rate of $10^{-3}$ and a batch size of 1000; neither network is refitted during ordering. Inputs are standardized for network training, but CCS uses curvature in the original data units at $t=0.05$ and held-out squared residuals from five-fold cross-fitted cubic splines of $X_j$ (six quantile knots; ridge penalty $10^{-3}$).

The DiffAN comparator uses masking without residue updates. For Gaussian data, it is trained separately with a 3000-epoch cap and the implementation's early-stopping and voting rules. For Gamma data, both ordering procedures share the same joint network, with DiffAN evaluated at $t=0.05$ without retraining. The methods differ in both marginal-score approximation and sink selection.

The Gamma joint network omits the first hidden-layer normalization to retain amplitude information relevant to score derivatives. We use paired Gaussian perturbations $\epsilon$ and $-\epsilon$. Half the noise levels come from $t\sim\mathrm{Uniform}(0,0.5)$ and half from noise standard deviations drawn log-uniformly on $[0.05,0.8]$. This mixture broadens the range of noise scales used for joint training. Projection uses $t\sim\mathrm{Uniform}(0,0.5)$.

\paragraph{Experimental setup.}
We compare Gaussian and Gamma conditional families at $d=30$ and $n=5000$ for both ER1 and ER3. A random node permutation defines the possible forward edges, of which exactly $kd$ are sampled for ER$k$. Each mechanism $f_j$ is drawn independently from a zero-mean Gaussian process with a multivariate RBF kernel of unit variance and unit length scale. In the Gaussian model, sources are standard normal and non-sources follow $X_j=f_j(X_{\pa_G(j)})+\varepsilon_j$ with independent $\varepsilon_j\sim\mathcal N(0,1)$. Gamma conditionals have fixed shape 6 and mean $2\exp\{0.5f_j(X_{\pa_G(j)})\}$; sources have the same shape and mean 2. Within each family and seed, both methods use the same graph and observations.

\begin{table}[H]
\TableStyle
\caption{Continuous-data ordering recovery ($d=30$, $n=5000$). Entries report $A_{\mathrm{top}}$ (higher is better) as means $\pm$ sample standard deviations.}
\label{tab:continuous-recovery}
\begin{tabular}{@{}llcc@{}}
\toprule
Graph & Method & Gaussian & Gamma \\
\midrule
\multirow{2}{*}{ER1} & \DISCOCont{} & $0.877\pm0.063$ & $0.640\pm0.126$ \\
    & DiffAN & $0.850\pm0.072$ & $0.540\pm0.137$ \\
\midrule
\multirow{2}{*}{ER3} & \DISCOCont{} & $0.866\pm0.041$ & $0.718\pm0.021$ \\
    & DiffAN & $0.882\pm0.049$ & $0.657\pm0.057$ \\
\bottomrule
\end{tabular}
\end{table}

\paragraph{Results.}
For Gaussian data, \DISCOCont{} has higher mean $A_{\mathrm{top}}$ on ER1, whereas DiffAN has higher mean $A_{\mathrm{top}}$ on ER3 (Table~\ref{tab:continuous-recovery}).

For Gamma data, the two methods share the same joint network. \DISCOCont{} has higher mean $A_{\mathrm{top}}$ on both ER1 and ER3. Both methods have lower ordering accuracy on ER1 than on ER3.

\subsection{Real Data: Lahman Batting Counts}
\label{app:lahman}

\paragraph{Data and preprocessing.}
We use the Lahman batting table through 2025 \citep{lahmanDatabase}. For each season from 2019 to 2025, we retain complete stints, select the 200 players with the most at-bats, aggregate retained stints by player, and pool the seven seasons, yielding $n=1{,}400$ player--season observations from 490 players on $d=17$ batting-count variables. The main analysis uses a four-level quantile representation (q4) to increase the number of observations per count level for CCS estimation. Each pooled column is cut at its empirical quartiles, with no additional support-rank recoding after binning. This support compression is a finite-sample preprocessing choice and need not preserve the conditional-independence structure of the raw counts; we therefore examine alternative support resolutions and raw counts below.

\paragraph{Order evaluation.}
Standard batting definitions \citep{mlbRecording} motivate the seven domain-informed directions
\[
\mathcal R_7 =
\{
\mathrm{AB}\!\to\!\mathrm{H},
\mathrm{AB}\!\to\!\mathrm{SO},
\mathrm{AB}\!\to\!\mathrm{GIDP},
\mathrm{H}\!\to\!\mathrm{2B},
\mathrm{H}\!\to\!\mathrm{3B},
\mathrm{H}\!\to\!\mathrm{HR},
\mathrm{BB}\!\to\!\mathrm{IBB}
\}.
\]
These accounting or containment relations are used only for post-fit diagnostics and are not supplied during fitting. Because several imply parent-dependent feasible ranges, they need not satisfy the parent-independent-support assumptions of Definition~\ref{def:ef-dag}. We therefore interpret $A_{\mathrm{top}}$ as agreement with a partial directional reference, not as DAG causal accuracy or as evidence that the Lahman distribution belongs exactly to the proposed model class.

For reference-set sensitivity, we additionally use $\mathcal R_{17}=\mathcal R_7\cup\mathcal R_{\mathrm{PP}}$, where
\[
\begin{aligned}
\mathcal R_{\mathrm{PP}}=\{&
\mathrm{G}\!\to\!\mathrm{H},
\mathrm{G}\!\to\!\mathrm{BB},
\mathrm{G}\!\to\!\mathrm{SO},
\mathrm{G}\!\to\!\mathrm{RBI},
\mathrm{AB}\!\to\!\mathrm{BB},
\mathrm{AB}\!\to\!\mathrm{RBI},\\
&
\mathrm{H}\!\to\!\mathrm{R},
\mathrm{H}\!\to\!\mathrm{SO},
\mathrm{HR}\!\to\!\mathrm{IBB},
\mathrm{SB}\!\to\!\mathrm{CS}
\}.
\end{aligned}
\]
The ten additional directions are extracted from the qualitative relationships discussed in the MLB analysis of \citet{park2019mrs}; neither $\mathcal R_7$ nor $\mathcal R_{17}$ is treated as a known causal DAG.

\paragraph{Method settings.}
General replication, evaluation, and estimator settings follow Appendix~\ref{app:experimental-setup}, and baseline implementations follow Appendix~\ref{app:disco-exp}. DISCO uses seeds 0--2, while the ODS family-specification analysis uses three paired cross-validation assignments shared across the Poisson, negative-binomial, and binomial fits so that only the supplied family specification changes. All fits use the same fixed cohort, so the reported variation reflects fitting or cross-validation variability rather than sampling uncertainty.

\paragraph{Family-specification sensitivity.}
Under the main q4 setting, the three \DISCO{} orders respect $6/7$, $7/7$, and $6/7$ relations in $\mathcal R_7$, yielding mean $A_{\mathrm{top}}=0.905$. The corresponding mean values for ODS under the Poisson, negative-binomial, and binomial specifications are $0.429$, $0.333$, and $0.429$, respectively. Figure~\ref{fig:lahman-family} summarizes ordering support and majority parent-edge recovery on these seven reference relations. Because the reference set is incomplete, non-reference edges are not interpreted as false causal edges.

\begin{figure}[t]
    \centering
    \includegraphics[width=.48\linewidth]{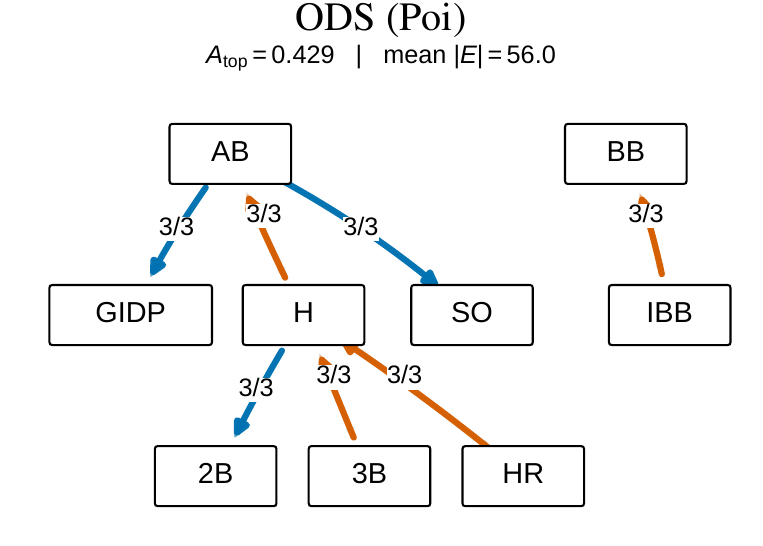}
    \hfill
    \includegraphics[width=.48\linewidth]{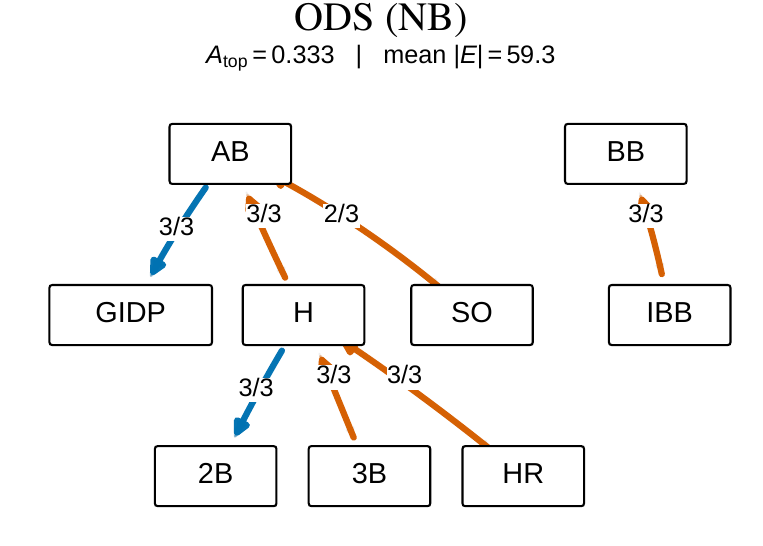}

    \vspace{0.6em}

    \includegraphics[width=.48\linewidth]{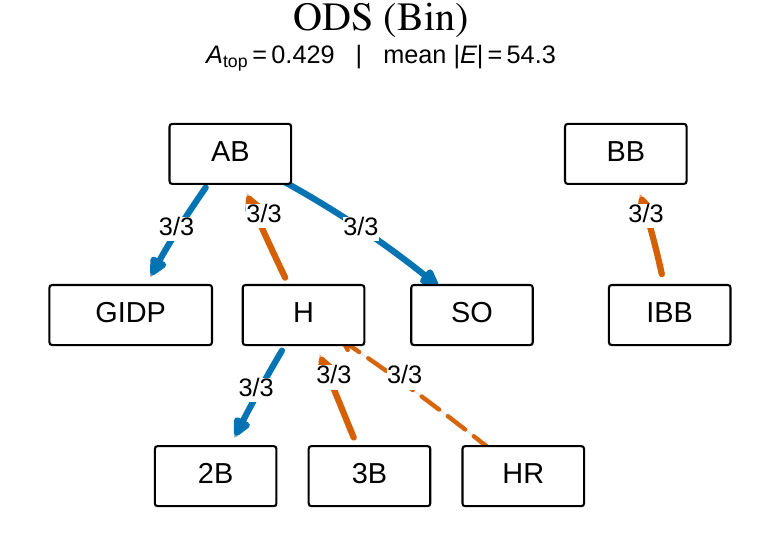}
    \hfill
    \includegraphics[width=.48\linewidth]{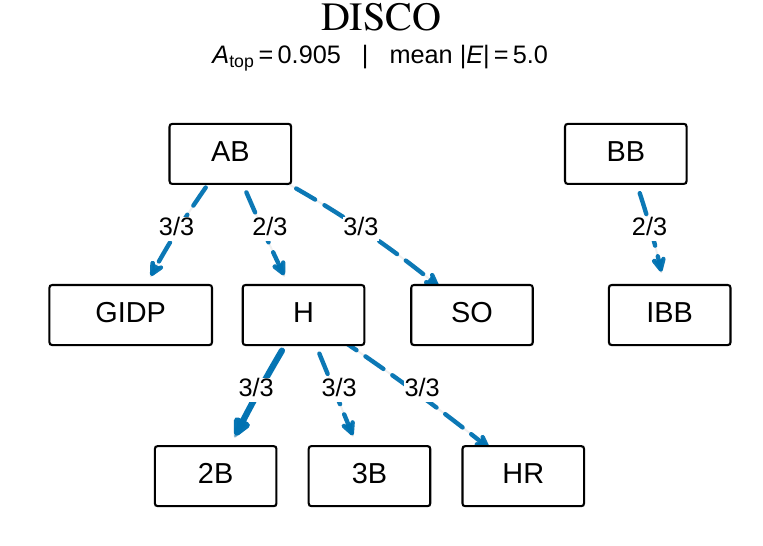}

    \caption{Ordering and parent-edge diagnostics on the Lahman 2019--2025 q4 representation. The top row shows ODS with Poisson and negative-binomial specifications; the bottom row shows ODS with a binomial specification and \DISCO{}. Blue arrows agree with the reference direction and orange arrows reverse it. Arrow labels report ordering support across the three runs. Solid arrows denote parent edges selected in at least two runs, whereas dashed arrows denote majority ordering without majority parent-edge selection. The reference pairs are used only for post-fit evaluation.}
    \label{fig:lahman-family}
\end{figure}

The three ODS fits are compared pairwise in Table~\ref{tab:lahman-ods-pairwise}. Order agreement is the fraction of unordered node pairs ranked in the same relative order by two fitted orders. Edge Jaccard is $|\widehat E_1\cap\widehat E_2|/|\widehat E_1\cup\widehat E_2|$ for their directed edge sets, treating opposite directions as distinct edges.

Poisson and binomial yield similar estimated orders, whereas comparisons involving the negative-binomial specification show larger changes in both ordering and edge set. Because the q4 codes are compressed counts rather than the original count scale, these results are interpreted as sensitivity to the supplied ODS specification rather than as correctly specified family comparisons.

\begin{table}[t]
\TableStyle
\caption{Pairwise sensitivity of ODS to the supplied conditional family on the Lahman q4 representation. Entries are mean $\pm$ standard deviation across three paired cross-validation assignments. Pairwise SHD compares two fitted directed graphs, not an estimated graph against a causal ground truth.}
\label{tab:lahman-ods-pairwise}
\begin{tabular}{@{}lccc@{}}
\toprule
ODS pair & Order agreement & Edge Jaccard & Pairwise SHD \\
\midrule
Poi--NB  & $0.882 \pm 0.007$ & $0.747 \pm 0.023$ & $16.7 \pm 1.5$ \\
Poi--Bin & $0.988 \pm 0.008$ & $0.924 \pm 0.021$ & $4.3 \pm 1.2$ \\
NB--Bin  & $0.880 \pm 0.004$ & $0.713 \pm 0.019$ & $19.0 \pm 1.0$ \\
\bottomrule
\end{tabular}
\end{table}

\paragraph{Boundary and support-resolution sensitivity.}
The main q4 implementation evaluates upper-grid shifts by clipping, whereas the population CCS is defined on the admissible finite-difference domain. As a reproduction check, the original q4 configuration again yields $6/7$, $7/7$, and $6/7$ agreement across the three runs, giving mean $A_{\mathrm{top}}^{\mathcal R_7}=0.905$. We then hold the \DISCO{} training configuration fixed and replace only curvature evaluation by the admissible-domain rule. Under this rule, $A_{\mathrm{top}}^{\mathcal R_7}=0.762$ for each of q4, q6, and q8, while the corresponding $\mathcal R_{17}$ values are $0.804$, $0.765$, and $0.784$. Thus the absolute q4 result is sensitive to boundary handling, whereas the admissible-domain ordering agreement is stable across the examined compressed support resolutions.

\begin{table}[t]
\TableStyle
\caption{Boundary and support-resolution sensitivity of \DISCO{} on Lahman. The first row reproduces the main q4 analysis; the remaining rows use admissible finite differences with the same \DISCO{} training configuration.}
\label{tab:lahman-boundary}
\begin{tabular}{@{}llrr@{}}
\toprule
Representation & Boundary & $A_{\mathrm{top}}^{\mathcal R_7}$ & $A_{\mathrm{top}}^{\mathcal R_{17}}$ \\
\midrule
q4  & clipped    & 0.905 & 0.843 \\
q4  & admissible & 0.762 & 0.804 \\
q6  & admissible & 0.762 & 0.765 \\
q8  & admissible & 0.762 & 0.784 \\
raw & admissible & 0.000 & 0.137 \\
\bottomrule
\end{tabular}
\end{table}

The substantially lower raw-count agreement indicates that support compression is important in this finite-sample Lahman analysis.

\paragraph{Preprocessing sensitivity.}
Because quantile compression changes the raw count mean--variance relationship used by QVF-based ODS, we additionally evaluate ODS on the uncompressed Lahman counts. For this sensitivity analysis, the Poisson and negative-binomial variants use log-link regressions, while the binomial variant uses a logit link with empirical nodewise support maxima. On the raw counts, $A_{\mathrm{top}}^{\mathcal R_7}$ is $0.333$, $0.857$, and $0.286$ for the Poisson, negative-binomial, and binomial specifications, respectively, with corresponding $\mathcal R_{17}$ values $0.412$, $0.843$, and $0.392$. Thus ODS can attain high directional agreement under one supplied family, but the result varies substantially with that specification; these results are therefore treated as preprocessing and family-specification sensitivity rather than as an oracle comparison with \DISCO{}.

\begin{table}[t]
\TableStyle
\caption{ODS sensitivity on the raw Lahman counts. The binomial row uses empirical nodewise support maxima and is not an oracle binomial specification.}
\label{tab:lahman-ods-raw}
\begin{tabular}{@{}lrrr@{}}
\toprule
Supplied family & $A_{\mathrm{top}}^{\mathcal R_7}$ & $A_{\mathrm{top}}^{\mathcal R_{17}}$ & Mean $|\widehat E|$ \\
\midrule
Poi & 0.333 & 0.412 & 73.7 \\
NB  & 0.857 & 0.843 & 69.3 \\
Bin & 0.286 & 0.392 & 79.7 \\
\bottomrule
\end{tabular}
\end{table}

\paragraph{Additional robustness checks.}
For the 2019--2025 q4 analysis, the three \DISCO{} orders agree with $\mathcal R_{17}$ on $15/17$, $15/17$, and $13/17$ relations, giving mean $A_{\mathrm{top}}=0.843$. Repeating the same cohort construction and q4 preprocessing for 2012--2018 gives mean agreement $0.952$ on $\mathcal R_7$ and $0.922$ on $\mathcal R_{17}$. These checks show that the reported directional pattern is not confined to one reference set or one analysis period, while remaining descriptive partial-order diagnostics rather than causal validation.

\section{Algorithm Diagnostics}
\label{app:estimator-diagnostics}

\subsection{Marginal Projection Methods}
\label{app:diagnostic-projection}

We compare three projection procedures on DAGs with $d=50$ and $n=5000$, using the same mechanisms for the three conditional families as in Appendix~\ref{app:diagnostic-parent-selection}. They share the data, three-fold splits, fixed joint networks, the first sink selected from all variables, and initial projection networks.

\paragraph{Additional training settings.}
We use the joint-network validation protocol, initial projection learning rate, and batch size specified in Appendix~\ref{app:experimental-setup}.

\emph{Amortized} projection uses this projection network throughout ordering without further training. After each removal, \emph{Stagewise direct} projection initializes the projection network with its previous parameters and trains it for 100 additional epochs, using predictions from the fixed joint network as regression targets. \emph{Stagewise recursive} projection uses the same training budget but obtains targets from the preceding projection network, whose parameters are held fixed. Both procedures use the joint network for the first reduced-set fit. Thus each stagewise procedure adds 48 fits per fold, for remaining-set sizes 49 through 2. Training targets and projection predictions use the same fully corrupted sample and noise level. No joint concrete-score network is refitted on a reduced set.

Parent selection uses the recovered orders with the same fixed amortized projection networks and default OCS rule, isolating the ordering procedure. Total computation budgets are not matched. Projection time includes fitting the shared initial projection network, the additional training at each removal, and saving network parameters, but excludes joint training, CCS evaluation, and parent selection.

\begin{table}[htbp]
\TableStyle
\caption{Comparison of marginal concrete-score projection methods ($d=50$, $n=5000$). Recovery metrics are means $\pm$ sample standard deviations, with seeds paired across methods within each conditional family. Projection time is the mean wall-clock time for the projection stage, in seconds.}
\label{tab:marginal-projection-comparison}
\begin{tabular*}{\linewidth}{@{\extracolsep{\fill}}llrrrr@{}}
\toprule
\shortstack{Conditional\\family} & Projection & $A_{\mathrm{top}}$ & $F_1$ & SHD & \shortstack{Projection\\time (s)} \\
\midrule
\multirow{3}{*}{Poi} & Amortized & $0.931\pm0.017$ & $0.892\pm0.034$ & $28.9\pm8.7$ & 92.6 \\
 & Stagewise direct & $0.917\pm0.016$ & $0.874\pm0.034$ & $34.0\pm9.2$ & 511.9 \\
 & Stagewise recursive & $0.916\pm0.013$ & $0.869\pm0.032$ & $35.1\pm8.2$ & 544.5 \\
\midrule
\multirow{3}{*}{NB} & Amortized & $0.931\pm0.017$ & $0.792\pm0.066$ & $51.2\pm15.4$ & 94.0 \\
 & Stagewise direct & $0.923\pm0.019$ & $0.785\pm0.055$ & $53.3\pm13.1$ & 515.0 \\
 & Stagewise recursive & $0.924\pm0.021$ & $0.788\pm0.053$ & $52.6\pm12.9$ & 545.2 \\
\midrule
\multirow{3}{*}{Bin} & Amortized & $0.901\pm0.027$ & $0.877\pm0.031$ & $32.3\pm8.4$ & 94.2 \\
 & Stagewise direct & $0.901\pm0.032$ & $0.873\pm0.036$ & $33.4\pm9.7$ & 510.3 \\
 & Stagewise recursive & $0.900\pm0.025$ & $0.874\pm0.035$ & $33.1\pm9.2$ & 542.0 \\
\bottomrule
\end{tabular*}
\end{table}

Amortized projection achieves comparable or higher mean $A_{\mathrm{top}}$, higher mean $F_1$, lower mean SHD, and shorter projection time in each conditional family. These differences are not uniform across seeds.

\subsection{Parent Selection}
\label{app:diagnostic-parent-selection}

We use the first Poi and NB mechanisms and the second Bin mechanism in Table~\ref{tab:dgp-spec}. The held-out folds collectively cover all 5000 observations, and no network is retrained for this comparison.

For the threshold-sensitivity analysis, the median/IQR standardization is unchanged. We vary $\tau_Z\in\{1,1.5,2,2.5,3\}$ and $\pi_{\min}\in\{2/3,1\}$: a candidate must satisfy the strict upper-tail test $Z_{j\to i}^{(b)}>\tau_Z$ in at least two or all three folds, respectively.

\begin{table}[H]
\TableStyle
\setlength{\tabcolsep}{3pt}
\caption{Parent-selection sensitivity for the threshold settings shown, with fixed learned orders and OCS estimates ($d=50$, $n=5000$). Entries are arithmetic means over the same seeds. Prec. and Rec. denote precision and recall against the full true DAG, including true edges excluded by the estimated order.}
\label{tab:parent-selection-sensitivity}
\begin{tabular*}{\linewidth}{@{\extracolsep{\fill}}l*{9}{r}@{}}
\toprule
& \multicolumn{3}{c}{Poi} & \multicolumn{3}{c}{NB}
& \multicolumn{3}{c}{Bin} \\
\cmidrule(lr){2-4}\cmidrule(lr){5-7}\cmidrule(lr){8-10}
Rule & $F_1$ & Prec. & Rec. & $F_1$ & Prec. & Rec.
& $F_1$ & Prec. & Rec. \\
\midrule
$\tau_Z=2,\ \pi_{\min}=1$ & 0.877 & 0.934 & 0.827 & 0.797 & 0.917 & 0.708 & 0.878 & 0.943 & 0.821 \\
$\tau_Z=2,\ \pi_{\min}=2/3$ & 0.891 & 0.898 & 0.884 & 0.870 & 0.898 & 0.845 & 0.890 & 0.895 & 0.885 \\
$\tau_Z=1,\ \pi_{\min}=1$ & 0.900 & 0.898 & 0.901 & 0.866 & 0.892 & 0.843 & 0.896 & 0.906 & 0.885 \\
$\tau_Z=1,\ \pi_{\min}=2/3$ & 0.800 & 0.709 & 0.917 & 0.813 & 0.735 & 0.910 & 0.809 & 0.732 & 0.907 \\
\bottomrule
\end{tabular*}
\end{table}

Table~\ref{tab:parent-selection-sensitivity} shows that lowering $\tau_Z$ from 2 to 1 while retaining all-fold agreement, or requiring two rather than three folds at $\tau_Z=2$, increases mean $F_1$ in all three conditional families. For negative binomial, relaxing fold agreement alone raises $F_1$ from 0.797 to 0.870. Relaxing both requirements increases false positives and lowers $F_1$ relative to the one-at-a-time changes. Thus the same OCS estimates can yield different recovery under different selection rules.
No threshold is selected or tuned against the true graph; this is not an error-control procedure and does not replace the default $\tau_Z=2$, $\pi_{\min}=1$.

\paragraph{Comparison of parent recovery methods.}
To separate parent selection from ordering, we compare the default \DISCO{} OCS selector with CAM pruning, an additive-regression variable-selection procedure \citep{buhlmann2014cam}, and PCM-GAM, a conditional-mean independence test \citep{lundborg2024projected}. Under a valid causal order in our semiparametric GLM model, $\E[X_i\mid X_{\mathrm{Pred}(i)}]=A_i'(\eta_i)$. Since $A_i''>0$, the nonconstant parent effects in Definition~\ref{def:ef-dag}, together with the support regularity conditions, make conditional-mean dependence on a predecessor equivalent to that predecessor being a parent of $i$. Thus the population criterion targeted by PCM identifies true parents; this does not guarantee exact recovery by the finite-sample PCM-GAM test.

We use a separate study with $d=50$ and $n=5000$, under both the true order and the same estimated \DISCO{} order. Family-specific generators are unchanged from the study above. Each dataset is split into 2500 ordering observations and 2500 parent-recovery observations; PCM further splits the latter into 1250 nuisance-training and 1250 testing observations. This disjoint-sample protocol is distinct from the cross-fitted threshold study above.

We reuse the CAM implementation described in Appendix~\ref{app:baseline-methods}, with cutoff $0.001$. PCM-GAM uses the public \texttt{comets::pcm} implementation of the projected covariance measure \citep{lundborg2024projected}, with GAM mean/projection regressions (basis parameter $k=4$) and a 500-tree random forest for the variance nuisance. Each candidate is tested at $\alpha=0.05$ without multiplicity correction. CAM and PCM-GAM receive raw counts without the true conditional family or link. These are method-specific settings, not tests at a common nominal level; using GAM nuisance estimates does not itself establish count-specific test calibration.

\begin{table}[htbp]
\TableStyle
\caption{Parent recovery under fixed orders ($d=50$, $n=5000$). Entries are arithmetic means over seeds shared by all methods and both order conditions. Metrics use the full true DAG, including edges excluded by an estimated order. Estimated denotes the shared \DISCO{} order; OCS denotes the default \DISCO{} parent selector.}
\label{tab:parent-method-comparison}
\begin{tabular*}{\linewidth}{@{\extracolsep{\fill}}lllrr@{}}
\toprule
\shortstack{Conditional\\family} & Order & Parent selector & $F_1$ & SHD \\
\midrule
\multirow{6}{*}{Poi} & \multirow{3}{*}{True} & OCS & 0.878 & 30.8 \\
 & & CAM & 0.989 & 3.2 \\
 & & PCM-GAM & 0.839 & 54.1 \\
 & \multirow{3}{*}{Estimated} & OCS & 0.858 & 37.2 \\
 & & CAM & 0.850 & 46.3 \\
 & & PCM-GAM & 0.712 & 105.6 \\
\midrule
\multirow{6}{*}{NB} & \multirow{3}{*}{True} & OCS & 0.780 & 50.5 \\
 & & CAM & 0.959 & 12.1 \\
 & & PCM-GAM & 0.832 & 56.3 \\
 & \multirow{3}{*}{Estimated} & OCS & 0.751 & 59.2 \\
 & & CAM & 0.823 & 54.9 \\
 & & PCM-GAM & 0.703 & 108.3 \\
\midrule
\multirow{6}{*}{Bin} & \multirow{3}{*}{True} & OCS & 0.820 & 43.1 \\
 & & CAM & 0.993 & 1.9 \\
 & & PCM-GAM & 0.827 & 58.3 \\
 & \multirow{3}{*}{Estimated} & OCS & 0.828 & 42.6 \\
 & & CAM & 0.898 & 28.7 \\
 & & PCM-GAM & 0.738 & 90.3 \\
\bottomrule
\end{tabular*}
\end{table}

With the true order, CAM has the highest mean $F_1$ and lowest SHD in every conditional family (Table~\ref{tab:parent-method-comparison}). With the shared estimated order, OCS outperforms PCM-GAM in all three families and has slightly higher $F_1$ and lower SHD than CAM for Poi; CAM remains strongest for NB and Bin. The OCS gap under the true order reflects limitations of learned scores and the practical selector, rather than the population OCS characterization; this comparison does not separate score-estimation error from thresholding error.


\clearpage
\begin{thebibliography}{99}

\bibitem[Blei et~al.(2003)]{blei2003lda}
D.~M. Blei, A.~Y. Ng, and M.~I. Jordan.
\newblock Latent {Dirichlet} allocation.
\newblock \emph{Journal of Machine Learning Research}, 3:993--1022, 2003.

\bibitem[Bodik and Chavez-Demoulin(2025)]{bodik2025cpcm}
J.~Bodik and V.~Chavez-Demoulin.
\newblock Identifiability of causal graphs under non-additive conditionally parametric causal models.
\newblock \emph{Journal of Machine Learning Research}, 26(264):1--55, 2025.

\bibitem[B\"uhlmann et~al.(2014)]{buhlmann2014cam}
P.~B\"uhlmann, J.~Peters, and J.~Ernest.
\newblock CAM: Causal additive models, high-dimensional order search and penalized regression.
\newblock \emph{The Annals of Statistics}, 42(6):2526--2556, 2014.

\bibitem[Castel et~al.(2015)]{castel2015allelic}
S.~E. Castel, A.~Levy-Moonshine, P.~Mohammadi, E.~Banks, and T.~Lappalainen.
\newblock Tools and best practices for data processing in allelic expression analysis.
\newblock \emph{Genome Biology}, 16:195, 2015.
\newblock doi:10.1186/s13059-015-0762-6.

\bibitem[Gao et~al.(2020)]{gao2020polynomial}
M.~Gao, Y.~Ding, and B.~Aragam.
\newblock A polynomial-time algorithm for learning nonparametric causal graphs.
\newblock In \emph{Advances in Neural Information Processing Systems 33 (NeurIPS)}, pages 11599--11611, 2020.

\bibitem[Ghoshal and Honorio(2018)]{ghoshal2018learning}
A.~Ghoshal and J.~Honorio.
\newblock Learning linear structural equation models in polynomial time and sample complexity.
\newblock In \emph{Proceedings of the 21st International Conference on Artificial Intelligence and Statistics (AISTATS)}, volume~84 of \emph{Proceedings of Machine Learning Research}, pages 1466--1475, 2018.

\bibitem[Huang(2017)]{huang2017cmp}
A.~Huang.
\newblock Mean-parametrized Conway--Maxwell--Poisson regression models for dispersed counts.
\newblock \emph{Statistical Modelling}, 17(6):359--380, 2017.
\newblock doi:10.1177/1471082X17697749.

\bibitem[Hyv\"arinen(2005)]{hyvarinen2005estimation}
A.~Hyv\"arinen.
\newblock Estimation of non-normalized statistical models by score matching.
\newblock \emph{Journal of Machine Learning Research}, 6(24):695--709, 2005.

\bibitem[Kang et~al.(2025)]{kang2025scino}
J.~Kang, S.~Kim, C.~Lee, D.~Hwang, J.~Chung, Y.~Ko, S.~Lee, S.~Kim, and S.~Lim.
\newblock Score-informed neural operator for enhancing ordering-based causal discovery.
\newblock In \emph{Advances in Neural Information Processing Systems 38 (NeurIPS)}, pages 113109--113151, 2025.

\bibitem[Lahman(n.d.)]{lahmanDatabase}
S.~Lahman.
\newblock Lahman baseball database.
\newblock Society for American Baseball Research.
\newblock \url{https://sabr.org/lahman-database/}, accessed September 22, 2026.

\bibitem[Lou et~al.(2024)]{lou2024discrete}
A.~Lou, C.~Meng, and S.~Ermon.
\newblock Discrete diffusion modeling by estimating the ratios of the data distribution.
\newblock In \emph{Proceedings of the 41st International Conference on Machine Learning (ICML)}, volume 235 of \emph{Proceedings of Machine Learning Research}, pages 32819--32848, 2024.

\bibitem[Love et~al.(2014)]{love2014deseq2}
M.~I. Love, W.~Huber, and S.~Anders.
\newblock Moderated estimation of fold change and dispersion for {RNA}-seq data with {DESeq2}.
\newblock \emph{Genome Biology}, 15(12):550, 2014.

\bibitem[Lundborg et~al.(2024)]{lundborg2024projected}
A.~R. Lundborg, I.~Kim, R.~D. Shah, and R.~J. Samworth.
\newblock The projected covariance measure for assumption-lean variable significance testing.
\newblock \emph{The Annals of Statistics}, 52(6):2851--2878, 2024.
\newblock doi:10.1214/24-AOS2447.

\bibitem[Lyu(2009)]{lyu2009generalized}
S.~Lyu.
\newblock Interpretation and generalization of score matching.
\newblock In \emph{Proceedings of the 25th Conference on Uncertainty in Artificial Intelligence (UAI)}, pages 359--366, 2009.

\bibitem[Major League Baseball, n.d.]{mlbRecording}
Major League Baseball.
\newblock Glossary of standard statistics: \href{https://www.mlb.com/glossary/standard-stats/at-bat}{At-bat}, \href{https://www.mlb.com/glossary/standard-stats/hit}{Hit}, and \href{https://www.mlb.com/glossary/standard-stats/intentional-walk}{Intentional walk}.
\newblock MLB.com, n.d. Accessed September 21, 2026.

\bibitem[McCullagh and Nelder(1989)]{mccullagh1989generalized}
P.~McCullagh and J.~A. Nelder.
\newblock \emph{Generalized Linear Models}.
\newblock Chapman \& Hall, 2nd edition, 1989.

\bibitem[Meng et~al.(2022)]{meng2022concrete}
C.~Meng, K.~Choi, J.~Song, and S.~Ermon.
\newblock Concrete score matching: Generalized score matching for discrete data.
\newblock In \emph{Advances in Neural Information Processing Systems 35 (NeurIPS)}, pages 34532--34545, 2022.

\bibitem[Montagna et~al.(2023a)]{montagna2023das}
F.~Montagna, N.~Noceti, L.~Rosasco, K.~Zhang, and F.~Locatello.
\newblock Scalable causal discovery with score matching.
\newblock In \emph{Proceedings of the 2nd Conference on Causal Learning and Reasoning (CLeaR)}, volume 213 of \emph{Proceedings of Machine Learning Research}, pages 752--771, 2023a.

\bibitem[Montagna et~al.(2023b)]{montagna2023nogam}
F.~Montagna, N.~Noceti, L.~Rosasco, K.~Zhang, and F.~Locatello.
\newblock Causal discovery with score matching on additive models with arbitrary noise.
\newblock In \emph{Proceedings of the 2nd Conference on Causal Learning and Reasoning (CLeaR)}, volume 213 of \emph{Proceedings of Machine Learning Research}, pages 726--751, 2023b.

\bibitem[Nelder and Wedderburn(1972)]{nelder1972generalized}
J.~A. Nelder and R.~W.~M. Wedderburn.
\newblock Generalized linear models.
\newblock \emph{Journal of the Royal Statistical Society: Series A (General)}, 135(3):370--384, 1972.

\bibitem[Park and Park(2019a)]{park2019ghd}
G.~Park and H.~Park.
\newblock Identifiability of generalized hypergeometric distribution (GHD) directed acyclic graphical models.
\newblock In \emph{Proceedings of the 22nd International Conference on Artificial Intelligence and Statistics (AISTATS)}, volume 89 of \emph{Proceedings of Machine Learning Research}, pages 158--166, 2019a.

\bibitem[Park and Park(2019b)]{park2019mrs}
G.~Park and S.~Park.
\newblock High-dimensional Poisson structural equation model learning via $\ell_1$-regularized regression.
\newblock \emph{Journal of Machine Learning Research}, 20(95):1--41, 2019b.

\bibitem[Park and Raskutti(2015)]{park2015poisson}
G.~Park and G.~Raskutti.
\newblock Learning large-scale Poisson DAG models based on overdispersion scoring.
\newblock In \emph{Advances in Neural Information Processing Systems 28 (NeurIPS)}, pages 631--639, 2015.

\bibitem[Park and Raskutti(2018)]{park2018qvf}
G.~Park and G.~Raskutti.
\newblock Learning quadratic variance function (QVF) DAG models via overdispersion scoring (ODS).
\newblock \emph{Journal of Machine Learning Research}, 18(224):1--44, 2018.

\bibitem[Peters et~al.(2014)]{peters2014cam}
J.~Peters, J.~M. Mooij, D.~Janzing, and B.~Sch\"olkopf.
\newblock Causal discovery with continuous additive noise models.
\newblock \emph{Journal of Machine Learning Research}, 15(58):2009--2053, 2014.

\bibitem[Rathouz and Gao(2009)]{rathouz2009generalized}
P.~J. Rathouz and L.~Gao.
\newblock Generalized linear models with unspecified reference distribution.
\newblock \emph{Biostatistics}, 10(2):205--218, 2009.

\bibitem[Robinson et~al.(2010)]{robinson2010edger}
M.~D. Robinson, D.~J. McCarthy, and G.~K. Smyth.
\newblock {edgeR}: a {Bioconductor} package for differential expression analysis of digital gene expression data.
\newblock \emph{Bioinformatics}, 26(1):139--140, 2010.

\bibitem[Rolland et~al.(2022)]{rolland2022score}
P.~Rolland, V.~Cevher, M.~Kleindessner, C.~Russell, D.~Janzing, B.~Schölkopf, and F.~Locatello.
\newblock Score matching enables causal discovery of nonlinear additive noise models.
\newblock In \emph{Proceedings of the 39th International Conference on Machine Learning (ICML)}, volume 162 of \emph{Proceedings of Machine Learning Research}, pages 18741--18753, 2022.

\bibitem[Sanchez et~al.(2023)]{sanchez2023diffan}
P.~Sanchez, X.~Liu, A.~Q. O'Neil, and S.~A. Tsaftaris.
\newblock Diffusion models for causal discovery via topological ordering.
\newblock In \emph{International Conference on Learning Representations (ICLR)}, 2023.

\bibitem[Stoklosa et~al.(2022)]{stoklosa2022negbin}
J.~Stoklosa, R.~V. Blakey, and F.~K.~C. Hui.
\newblock An overview of modern applications of negative binomial modelling in ecology and biodiversity.
\newblock \emph{Diversity}, 14(5):320, 2022.

\bibitem[Vo et~al.(2026)]{vo2026generalized}
V.~Vo, T.~Le, H.~Zhao, E.~V. Bonilla, and D.~Phung.
\newblock Ordering-based causal discovery via generalized score matching.
\newblock In \emph{Proceedings of the 32nd ACM SIGKDD Conference on Knowledge Discovery and Data Mining V.2 (KDD)}, pages 4706--4717, 2026.
\newblock doi:10.1145/3770855.3817672.

\bibitem[Yang et~al.(2018)]{yang2018semiparametric}
Z.~Yang, Y.~Ning, and H.~Liu.
\newblock On semiparametric exponential family graphical models.
\newblock \emph{Journal of Machine Learning Research}, 19(57):1--59, 2018.

\bibitem[Zhang et~al.(2021)]{zhang2021gcastle}
K.~Zhang, S.~Zhu, M.~Kalander, I.~Ng, J.~Ye, Z.~Chen, and L.~Pan.
\newblock {gCastle}: A Python toolbox for causal discovery.
\newblock \emph{arXiv preprint arXiv:2111.15155}, 2021.

\bibitem[Zheng et~al.(2020)]{zheng2020learning}
X.~Zheng, C.~Dan, B.~Aragam, P.~Ravikumar, and E.~P. Xing.
\newblock Learning sparse nonparametric DAGs.
\newblock In \emph{Proceedings of the 23rd International Conference on Artificial Intelligence and Statistics (AISTATS)}, volume 108 of \emph{Proceedings of Machine Learning Research}, pages 3414--3425, 2020.

\end{thebibliography}
\end{document}